\documentclass{article}

\PassOptionsToPackage{numbers, sort}{natbib}
\PassOptionsToPackage{dvipsnames,table}{xcolor}

\usepackage[final]{neurips_2022}

\usepackage{wrapfig}
\usepackage{marginnote}
\usepackage[utf8]{inputenc} 
\usepackage[T1]{fontenc}    
\usepackage{hyperref}       
\usepackage{url}            
\usepackage{booktabs}       
\usepackage{amsfonts}       
\usepackage{nicefrac}       
\usepackage{microtype}      
\usepackage{xcolor}         

\usepackage{neurips_2022}

\usepackage[hashEnumerators,smartEllipses]{markdown}
\usepackage[utf8]{inputenc} 
\usepackage[T1]{fontenc}    
\usepackage{hyperref}       
\usepackage{url}            
\usepackage{booktabs}       
\usepackage{amsfonts}       
\usepackage{nicefrac}       
\usepackage{microtype}      
\usepackage{xcolor}         

\usepackage{multirow}
\usepackage{amsfonts}
\usepackage{epsfig}
\usepackage{amsmath}
\usepackage{esvect}
\usepackage{amssymb}
\usepackage{bbm}
\usepackage{amsthm}
\usepackage{comment}
\usepackage[english]{babel}
\usepackage{forloop}

\renewcommand\citet[1]{\citeauthor{#1}~\cite{#1}} 
\usepackage{color}
\usepackage{algorithm}
\usepackage{algpseudocode}
\usepackage{float}

\usepackage{nccmath}
\usepackage{pgfplots}
\usepackage{booktabs}
\usepackage{listings}
\usepgfplotslibrary{external}

\usepackage{wrapfig}
\usepackage{enumitem}
\usepackage{url}
\usepackage{etoolbox}
\usepackage{subfigure}
\usepackage{longtable}
\usepackage{multirow}
\usepackage{sidenotes}

\DeclareMathOperator*{\argmax}{arg\,max}
\DeclareMathOperator*{\argmin}{arg\,min}

\newtheorem{theorem}{Theorem}[section]

\usepackage{siunitx}
\usepackage{mathtools}
\usepackage{stmaryrd}
\usepackage{mathrsfs}
\usepackage{amssymb}

\newcommand{\ncomponents}{m}
\newcommand{\qfunc}{$Q$-function} 
\newcommand{\qfuncs}{\qfunc s}

\renewcommand{\eqref}[1]{Eq.~\ref{#1}}
\newcommand{\figref}[1]{Fig.~\ref{#1}}

\newcommand{\margincomment}[2]{  
    \sidenote[]{
        {\footnotesize 
            \textbf{#1}: #2 }
        }
}

\def\comments{0}

\ifx\comments\undefined
    \def\comments{0}
\fi
\if\comments1
    \newcommand{\inlinecomment}[2]{\textit{\color{red} \textbf{#1}: #2 }}
    \newcommand{\commentjm}[1]{ \margincomment{JM}{#1} }
    \newcommand{\commentjmi}[1]{ \inlinecomment{JM}{#1} }
    \newcommand{\commentcs}[1]{ \margincomment{CS}{#1} }
    \newcommand{\commentcsi}[1]{ \inlinecomment{CS}{#1} }
    
    \newcommand{\commentts}[1]{ \margincomment{TS}{#1} }
    \newcommand{\commenttsi}[1]{ \inlinecomment{TS}{#1} }
    \newcommand{\commentea}[1]{ \margincomment{EA}{#1} }
    \newcommand{\commenteai}[1]{ \inlinecomment{EA}{#1} }
     
    \newcommand{\commentad}[1]{ \margincomment{AD}{#1} }
    \newcommand{\commentadi}[1]{ \inlinecomment{AD}{#1} }
\else
    \newcommand{\inlinecomment}[2]{}
    
    \newcommand{\commentjm}[1]{ }
    \newcommand{\commentjmi}[1]{ }
    \newcommand{\commentad}[1]{  }
    \newcommand{\commentadi}[1]{  }
    \newcommand{\commentcs}[1]{ }
    \newcommand{\commentcsi}[1]{ }
    \newcommand{\commentts}[1]{ }
    \newcommand{\commenttsi}[1]{ }
    \newcommand{\commentea}[1]{ }
    \newcommand{\commenteai}[1]{  }
\fi

\usepackage{etoolbox}
\makeatletter
\patchcmd\@acf{\hskip\z@}{}{}{}
\patchcmd\@acf{\hskip\z@}{}{}{}
\makeatother

\usepackage{pgfplotstable}
\pgfplotsset{compat=1.8}

\title{Value Function Decomposition for Iterative Design of Reinforcement Learning Agents}

\makeatletter

\def\@fnsymbol#1{\ensuremath{\ifcase#1\or \dagger\or *\or \ddagger\or
   \mathsection\or \mathparagraph\or \|\or **\or \dagger\dagger
   \or \ddagger\ddagger \else\@ctrerr\fi}}

\author{
James MacGlashan\thanks{Sony AI} \\
\texttt{james.macglashan@sony.com}
\And
Evan Archer$^\dagger$\thanks{Equal contribution} \\
\texttt{evan.archer@sony.com}
\And
Alisa Devlic$^{\dagger *}$ \\
\texttt{alisa.devlic@sony.com}
\And
Takuma Seno$^{\dagger *}$ \\
\texttt{takuma.seno@sony.com}
\And
Craig Sherstan$^{\dagger *}$ \\
\texttt{craig.sherstan@sony.com}
\AND
Peter R. Wurman$^\dagger$ \\
\texttt{peter.wurman@sony.com}
\And
Peter Stone$^\dagger$\thanks{The University of Texas at Austin} \\
\texttt{pstone@cs.utexas.edu}
}

\makeatother

\newcounter{subfigurenumber}

\begin{document}

\maketitle

\begin{abstract}
    Designing reinforcement learning (RL) agents is typically a difficult process that requires numerous design iterations. Learning can fail for a multitude of reasons, and standard RL methods provide too few tools to provide insight into the exact cause.
    In this paper, we show how to integrate \textit{value decomposition} into a broad class of actor-critic algorithms and use it to assist in the iterative agent-design process. 
    Value decomposition separates a reward function into distinct components and learns value estimates for each. These value estimates provide 
    insight into an agent's learning and decision-making process and 
    enable new training methods to mitigate common problems. 
    As a demonstration, we introduce SAC-D, a variant of soft actor-critic (SAC) adapted for value decomposition. SAC-D maintains similar performance to SAC, while learning a larger set of value predictions.
    We also introduce decomposition-based tools that exploit this information, including a new reward
    \textit{influence} metric, which measures each reward component's effect on agent decision-making. Using these tools, we provide several demonstrations of decomposition's use in identifying and addressing problems in the design of both environments and agents.
    Value decomposition is broadly applicable and easy to incorporate into existing algorithms and workflows, making it a powerful tool in an RL practitioner's toolbox.

\end{abstract}

\section{Introduction}
Deep reinforcement-learning (RL) approaches have achieved successes in a range of application areas such as gaming (\cite{silver2017mastering, wurman2022GT, berner2019dota, vinyals2019grandmaster}), robotics (\cite{levine2016end}), and the natural sciences (\cite{mills2020finding, tunyasuvunakool2021highly}). Despite these successes, applying RL techniques to complex control problems remains a daunting undertaking, where initial attempts often result in underwhelming performance. 
Unfortunately, there are many reasons why an agent may fail to learn a good policy, making it difficult to diagnose which reason(s) caused a particular agent to fail.
For example: an agent may fail because the state features were insufficient to make accurate predictions, different task objectives defining the reward function may be imbalanced, the agent may fail to sufficiently explore the state-action space, values may not accurately propagate to more distant states, the neural network may not have sufficient capacity to approximate the policy or value function(s), or, there may be subtle differences between training and evaluation environments. Without a way to diagnose the causes of poor performance or to recognize when a problem has been remedied, practitioners typically engage in a long trial-and-error design process until an agent reaches a desired level of performance.
Frustrations with this trial-and-error process 
have been expressed in other work~\cite{hayes_practical_2022}.

We describe how \textit{value decomposition}, a simple, broadly-applicable technique, can address these application challenges.
In RL, the agent receives a reward that is often a sum of many reward components, each designed to encode some aspect of the desired agent behavior. From this \textit{composite reward}, it learns a single \textit{composite value function}.
Using value decomposition, an agent learns a \textit{component value function} for each reward component.
To perform policy optimization, the composite value function is recovered by taking a weighted sum of the component value functions. While prior work has proposed value decomposition methods for discrete-action Q-learning ~\cite{Juozapaitis2019ExplainableRL, Russell2003QDecompositionFR, HybridRA},
we show how value decomposition can be incorporated into a broad class of actor-critic (AC) methods. In addition, we introduce SAC-D, a version of soft actor-critic (SAC)~\cite{haarnoja2018soft,sac} with value decomposition, and explore its use in multi-dimensional continuous-action environments.
We also introduce the \textit{influence} metric, which measures how much an agent's decisions are affected by each reward component.

While earlier work focuses on its use in reward design~\cite{Juozapaitis2019ExplainableRL,hayes_practical_2022}, value decomposition can facilitate diagnosis of a wide range
of issues and enable new training methodologies.
To demonstrate its utility, in Sec.~\ref{sec:analysis_examples} we show how to use it to: (1) diagnose insufficient state features; (2) diagnose value prediction errors and exploit the decomposed structure to inject background-knowledge; and (3) identify reward components that are inhibiting exploration and mitigate the effect by gradually incorporating component predictions into policy optimization.

Value decomposition's additional diagnostic and training capabilities come at the cost of a more-challenging prediction problem: instead of learning a single value function, many must be learned.
To investigate if this difficulty negatively impacts agent performance, we compare the average performance of SAC-D to SAC on benchmark environments. We find that a naive implementation of SAC-D underperforms SAC
and then show how to improve SAC-D so that it matches and sometimes exceeds SAC's performance. These improvements may also be applied to value decomposition for other AC algorithms.

While variations of value decomposition has been explored extensively in past work (see Sec.~\ref{sec:related}), in this paper, we make the following contributions.
(1) We show how to integrate value decomposition into a broad class of actor-critic algorithms. 
(2) We analyze the performance of different implementations of value decomposition for SAC on a range of benchmark continuous-action environments. (3) We introduce the \textit{influence} metric: a novel value decomposition metric for measuring how much each reward component affects decision-making. (4) We provide a set of illustrative examples of how value decomposition and influence can be used to diagnose various kinds of learning challenges. (5) We describe new training methods that exploit the value decomposition structure and can be used to mitigate different learning challenges.

\section{Background}
\subsection{MDPs and Q-functions}
In RL, an agent's interaction with the environment is modeled as a Markov Decision Process (MDP): $(S, A, P, R, \gamma)$, where $S$ is a set of states, $A$ is a set of actions, $P : S \times A \times S \rightarrow \mathbb{R}$ is a state transition probability function $P(s, a, s')=\Pr(S_{t+1}=s'|S_{t}=s, A_{t}=a)$, $R: S \times A \rightarrow \mathbb{R}$ is a reward function $R(s, a)=E[R_{t+1}|S_{t}=s, A_{t}=a]$, and $\gamma \in [0, 1]$ discounts future rewards.\footnote{For continuing tasks, $\gamma$ must be $< 1$, and we only consider algorithms for $\gamma < 1$.}
The goal of an agent is to learn a policy $\pi(a|s)$ that maps states to an action probability distribution that maximizes the sum of future rewards. The agent is trained to maximize the discounted return $E \left[ \sum_t^\infty \gamma^t R(s_t, a_t) \right]$.
The \qfunc{} maps state-action pairs to the expected cumulative discounted reward when starting in state $s$, taking action $a$, and then following policy $\pi$ thereafter: 
\begin{equation}
\label{eq:qdef}
  Q^\pi(s, a) \triangleq \mathbb{E}\left[\sum_{t=0}^\infty \gamma^{t} R(s_t, a_t) | \pi, s_0=s, a_0=a \right].
\end{equation}

\subsection{Soft actor-critic}
Soft actor-critic (SAC)~\cite{haarnoja2018soft,sac} is an off-policy actor-critic algorithm parameterized with five neural networks: a stochastic policy network $\pi$ with parameters $\phi$, and two pairs of \qfuncs{} and target \qfuncs{} with 
parameters ($\theta_1$, $\theta_2$) and ($\bar{\theta_1}$, $\bar{\theta_2}$), respectively. As with other actor-critic algorithms, SAC has two main steps: policy evaluation (in which it estimates the Q-function for policy $\pi$), and policy improvement (in which it optimize the policy to maximize its Q-function estimates). Unlike other actor-critic algorithms, SAC optimizes a maximum entropy formulation of the MDP, in which rewards are augmented with policy entropy bonuses that prevent premature policy collapse. 
To perform policy evaluation and improvement, SAC minimizes the following loss functions simultaneously:
\begin{subequations}
    \begin{align}
        L_{Q_i}  &= \mathbb{E}\left[ \frac{1}{2} \left( Q(s, a; \theta_i) - y \right)^2 \right]~\text{for}~i\in \left\{1,2\right\},\\
        L_\pi    &= \mathbb{E}\left[\alpha \log\pi(u | s; \phi) - \min_{j\in \{1,2\}} Q(s, u; \theta_j) \right],
    \end{align}
    \label{eq:sac_eqs}
\end{subequations}
where $(s, a, r, s')$ transitions are drawn from an experience replay buffer, $y :=  r + \gamma \left(\min_{j\in \{1,2\}} Q(s', a'; \bar{\theta_j}) - \alpha\log\pi(a' | s'; \phi) \right)$,
$a' \sim \pi(\cdot|s';\phi)$, $u \sim \pi(\cdot|s;\phi)$, and $\alpha$ is an (optionally learned) entropy regularization parameter. 
The $\min$ of \qfunc{} pairs addresses overestimation bias in value function estimation~\cite{hasselt2010double,fujimoto2018addressing}. 
The parameters $\bar{\theta_1}$ and $\bar{\theta_2}$ are updated toward $\theta_1$ and $\theta_2$ via an exponentially moving average each step.

\subsection{Environments}
\label{sec:sample_env}
Throughout the paper, we italicize descriptive names for the components of the continuous-action Lunar Lander (LL), Bipedal Walker (BW) and Bipedal Walker Hardcore (BWH)  environments. In LL, an agent must land a spacecraft in the center of a landing zone using as little fuel as possible. The reward components include: a reward for successful \textit{landing}; penalties for crashing (\textit{crash}) and engine usage (\textit{main}, \textit{side}); shaping rewards used to encourage the agent to stay upright (\textit{angle}), move towards the center of the landing pad (\textit{position}) with low velocity (\textit{velocity}) and land with both legs (\textit{right leg}, \textit{left leg}). In BW, an agent learns to make a 2-legged robot walk. The reward components include: a reward for \textit{forward} progress, a penalty for falling (\textit{failure}), a cost for actions (\textit{control}), and a shaping reward to discourage head movement (\textit{head}). BWH is identical to BW, but adds additional obstacles for the agent to navigate.

\section{Value decomposition for actor-critic methods}
Most RL algorithms 
estimate the value function and use it to improve its policy. Unfortunately, value functions and policies provide little insight into the agent's decision-making.
However, reward functions are often \textit{composite} functions of multiple \textit{component} state-action signals. By learning a value function estimate for the current policy for each component, practitioners gain insight into what the agent expects to happen 
and how these reward components interact. Naturally, policy improvement still requires the composite value function. 
In Sec.~\ref{sec:compositeQ} we show how the composite Q-function can be recovered from the component Q-functions. From this property, a range of actor-critic algorithms can be adapted to use value decomposition by following the below template.\footnote{Composite state value functions can similarly be recovered; however, Q-functions allow for deeper introspection, so we focus on that setting in this work.} 
\begin{enumerate}
\item Alter Q-function networks to have $m$ outputs instead of $1$, where $m$ is the number of reward components.
\item Use the base algorithm's Q-function update for each of the $m$ components, replacing the composite reward term with the respective component reward term.
\item Apply the base algorithm's policy improvement step by first recovering the composite Q-function.
\end{enumerate}
For example, this template can be applied to algorithms that use TD(0)~\cite{sutton2018reinforcement}, or ones that use Retrace~\cite{munos2016safe}. It works with algorithms that improve the policy by differentiating through the Q-function~\cite{sac,Lillicrap2016continuous,fujimoto2018addressing}, and ones that fit it toward non-parametric target action distributions~\cite{abdolmaleki2018relative}.

Although this template is conceptually simple, learning component Q-functions poses a more difficult prediction problem: multiple predictions must be learned instead of one composite prediction. Ideally, this increased difficulty would not negatively impact agent performance. In Sec.~\ref{sec:sac-d} we introduce SAC-D, an adaptation of SAC to use value decomposition, and describe additions we made to the above template to maintain performance with conventional SAC. Although these additions
are contextualised to SAC, they are general and can be used when adapting other actor-critic algorithms.

\subsection{Recovering the composite Q-function}
\label{sec:compositeQ}

We assume the environment's reward function is a linear combination of $\ncomponents$ components: $R(s, a) \triangleq \sum_i^\ncomponents w_i R_i(s, a)$, where $w_i \in \mathbb{R}$ is a scalar \textit{component weight} for the $i$th component and $R_i(s, a) \rightarrow \mathbb{R}$ is the reward function of the $i$th component for state-action pair $\left(s, a\right)$. Applying the linearity of expectation, we find the $Q$-function inherits the linear structure from the reward\footnote{See Theorem~\ref{thm:valdecomp} for the proof. This linear decomposition property of value functions has been explored elsewhere~\cite{barreto2017successor,dayan1993improving}, but in different contexts and with different motivations. See Sec.~\ref{sec:related} for more information.}: 
\begin{equation}
\label{eq:recover}
    Q^\pi(s, a) = \sum_i^\ncomponents w_i Q^\pi_i(s, a),
\end{equation}
where we define the $i$'th \textit{component \qfunc{}} as $Q^\pi_i(s, a) \triangleq E[\sum_t \gamma^t R_i(s_t, a_t) | \pi, s_0=s, a_0=a]$. Unless otherwise specified, we assume $w_i = 1$ for all $i$. 
Because the component weights are factored out of the component Q-functions, they may be varied without changing the component prediction target, allowing the policy to be evaluated for any weight combination. 
Although the assumption of linearity may seem restrictive, note that each reward component may be a non-linear function of state variables, allowing for very expressive environment rewards. Furthermore, many environments, including all the environments we investigate in this paper, are naturally structured as a sum of (non-linear) reward components.

\subsection{SAC with value decomposition}
\label{sec:sac-d}
\begin{algorithm}[t]
\footnotesize
\caption{SAC-D and SAC-D-CAGrad Update}
\label{alg:sac_d}
\begin{algorithmic}[1]
\Require Experience replay buffer $B$; twin \qfunc{} parameters  $\theta_1, \theta_2$ (with $\Theta = \theta_1 \cup \theta_2$) and target parameters $\bar{\theta_1},\bar{\theta_2}$; policy parameters $\phi$; discount factor $\gamma$; entropy parameter $\alpha$; reward weights $w \in \mathbb{R}^{m+1}$; learning rates $\lambda_q, \lambda_\pi$; target network step size $\eta$; Boolean use\_cagrad for SAC-D-CAGRAD or SAC-D.
\State Sample transition (minibatch) $(s, a, r, s') \sim B$ \algorithmiccomment{$r \in \mathbb{R}^\ncomponents$ is a vector of $\ncomponents$ reward components}
\State Sample policy actions $a' \sim \pi(\cdot | s'; \phi)$ and $u \sim \pi(\cdot| s; \phi)$
\State $r_{m+1} \gets \gamma \alpha \log\pi(a' | s'; \phi)$ \label{alg1:line:ent_reward} \algorithmiccomment{Extend reward vector to include entropy reward}
\State $j \gets \underset{j \in \{1, 2\}}{\arg\min} \sum_i^{\ncomponents+1} w_i Q_i(s', a'; \bar{\theta_j})$ \algorithmiccomment{Find target network by minimum composite $Q$-value}
\State $y_{i} \gets r_{i} + \gamma Q_i(s', a'; \bar{\theta}_{j})$ \label{alg1:line:y}
\State $LQ_i \gets \sum_{j=1}^2 \frac{1}{2} \left( Q_i(s, a; \theta_j) - y_i \right)^2$ \label{alg1:line:LQ}
\State $L\pi \gets \alpha \log\pi(u | s; \phi) - \min_{j\in \{1,2\}} \sum_i^{m+1} w_i Q_i(s, u; \theta_j)$ 
\If{\texttt{use\_cagrad}}
  \State $\Theta \gets \Theta - \lambda_q$\Call{CAGrad}{$\mathbf{J}_{LQ}, \Theta$} 
  \label{alg1:line:cagrad}
\Else
  \State $\Theta \gets \Theta - \lambda_q \nabla_\Theta \frac{1}{m+1} \sum_i^{m+1} LQ_i$ 
\EndIf
\State $\phi \gets \phi - \lambda_\pi \nabla_\phi L\pi$
\State Update target networks $\bar{\Theta} \gets (1 - \eta)\bar{\Theta} + \eta \Theta$
\end{algorithmic}
\end{algorithm}
Here we introduce SAC-D (Alg.~\ref{alg:sac_d}), an adaptation of SAC to use value decomposition.
Adapting SAC only requires one additional consideration from our template: the entropy bonus reward term SAC adds is treated as an $m+1$'th reward component: $R_{m+1}(s') \triangleq \gamma \alpha \log \pi(a' |s'; \phi)$ (line~\ref{alg1:line:ent_reward}). However, this approach, which we refer to as \textit{SAC-D-Naive}, underperforms SAC in many settings. We found two additional modifications essential to match the performance of SAC. The first 
concerns how we apply the twin-network minimum of Eq.~\ref{eq:sac_eqs} in the context of value decomposition. 
The second 
is to use \textbf{C}onflict-\textbf{A}verse \textbf{Grad}ient descent (CAGrad)~\cite{cagrad} to address optimization problems that arise when training multi-headed neural networks.
We refer to SAC with value decomposition and the twin-network correction as \textit{SAC-D}, and the variant with twin-network correction and CAGrad as \textit{SAC-D-CAGrad}.

\textbf{Twin-network minimums in value decomposition:}
In SAC, the Q-value target is the minimum of two Q-function networks (\eqref{eq:sac_eqs}). Using the same Q-value update rule for each component, as described in our template, suggests using a minimum for each component target: $q_i := \min_{j \in \{1,2\}} Q_i(s, a; \bar{\theta_j})$. However, this is not a good choice in practice. The purpose of the twin-network minimum is to mitigate overestimation bias from the feedback loop of the policy optimizing the Q-function. Because the policy optimizes the composite \qfunc, a better approach is to use all the predictions from the network with the minimum composite Q-function (Alg.~\ref{alg:sac_d}, lines~\ref{alg1:line:y}-\ref{alg1:line:LQ})). This approach reduces underestimation bias and improves performance compared to an element-wise minimum (see Sec.~\ref{sec:robustness}).

\textbf{Mediating the difficulty of multi-objective optimization:}
Even though the scalar values of the composite \qfunc{} are identical to those used in SAC, simultaneous optimization of all $Q_i^\pi$ components may introduce training problems common in multi-objective optimization: conflicting gradients, high curvature and large differences in gradient magnitudes~\citep{yu2020gradient}.\footnote{Multi-headed prediction can also improve representation learning, as in work on auxiliary tasks~\cite{unreal,learning_to_navigate}}
The CAGrad method, designed for the multi-task RL setting, addresses these issues by replacing the gradient of a multi-task objective with a weighted sum of per-task loss gradients. This updated gradient step maximizes the improvement of the worst-performing task on each optimization step, and still converges to a minimum of the unmodified loss. We incorporate CAGrad into SAC-D by treating each component as a ``task'', and update the gradient vector accordingly 
(Alg.~\ref{alg:sac_d}, line~\ref{alg1:line:cagrad}). 

\begin{figure}[t]
    \centering
    \subfigure{\includegraphics[width=.95\textwidth]{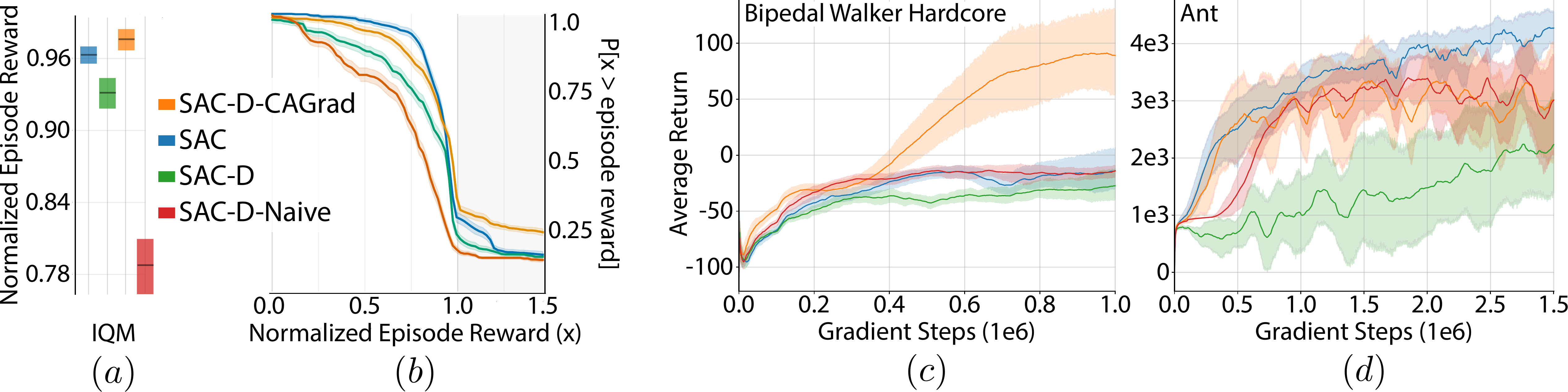}\label{fig:robustness_aggregate_metrics}}
    \subfigure{\label{fig:robustness_score_distribution}}
    \subfigure{\label{fig:bwh}}
    \subfigure{\label{fig:ant}}
    \caption{\textbf{Benchmark results}
    \textbf{(a)}-\textbf{(b)} Comparison of algorithms over 8 environments. For each algorithm, scores are collected from the 10 final checkpoints (10 episodes $\times$ 10 final policies), and normalized on individual environments. The 95\% CIs are estimated using stratified bootstrap sampling.
    \textbf{(a)} The interquartile mean (IQM), which gives the mean over the central 50\% of values.
    \textbf{(b)} For each normalized score $x$, \% of runs achieving a score of at least $x$. Grey region, $x>1.0$, corresponds to runs outperforming the mean SAC baseline score. 
    \textbf{(c)} Training curves on continuous-control benchmarks for BWH. Solid lines represent mean episode return over 10 trials. Shaded areas are confidence intervals. Lines are uniformly smoothed for clarity.
    \textbf{(d)} Same as \textbf{(c)}, but for Ant.
    }
    \label{fig:robustness_learning_curves}
\end{figure}
\section{Robustness experiments}
\label{sec:robustness}

We benchmark SAC-D, SAC-D-CAGrad and SAC-D-Naive against SAC on a selection of continuous-action Gym~\cite{brockman2016openai} environments. 
For each environment, we exposed existing additive reward components \textit{without} altering the behavior of the environments or their composite rewards. That is, these environments already implemented their reward functions as a linear combination of separate reward components and we simply exposed that information to the algorithm (for details, see App.~\ref{appendix:environment}). As outlined in App.~\ref{appendix:detail}, we used hyperparameters previously published for use with SAC~\cite{sac} for all experiments. We tied SAC-D's hyperparameters to SAC's because our goal is for value decomposition to be a drop-in addition without significant loss in agent-performance. However, it is possible better performance could be reached with tuning.

Figure~\ref{fig:robustness_aggregate_metrics} shows the performance of each algorithm aggregated across all environments and all experimental runs. Figure~\ref{fig:robustness_score_distribution} shows the same information, but highlights the distribution of scores across experimental runs.  
In aggregate, SAC-D-CAGrad slightly outperforms SAC, although it has a broader range of performance scores. SAC-D-Naive significantly underperforms SAC.

We provide training curves for the 8 environments investigated in App.~\ref{appendix:training_curves}, but highlight the atypical training curves for BWH (\figref{fig:bwh}) and Ant (\figref{fig:ant}).
In the case of BWH, SAC-D-CAGrad significantly outperforms SAC. It was not the goal of this work for SAC-D to improve on SAC. Rather, we sought to provide more insights into the learning process. As such, we make no strong claims about when SAC-D can be expected to outperform SAC. Nevertheless, this result does suggest that SAC-D may sometimes benefit from auxiliary task learning. We leave this question as a subject for future investigation.

In Ant, SAC-D (without CAGrad) underperforms all other methods. We found that infrequent environment termination causes large Q-function errors and leads to catastrophic gradient conflicts. Further analysis of termination issues and CAGrad behavior is described in Appendices~\ref{appendix:ant_analysis} and~\ref{appendix:cagrad_analysis} respectively.

\begin{figure}[t]
\centering
\includegraphics[width=0.75\columnwidth]{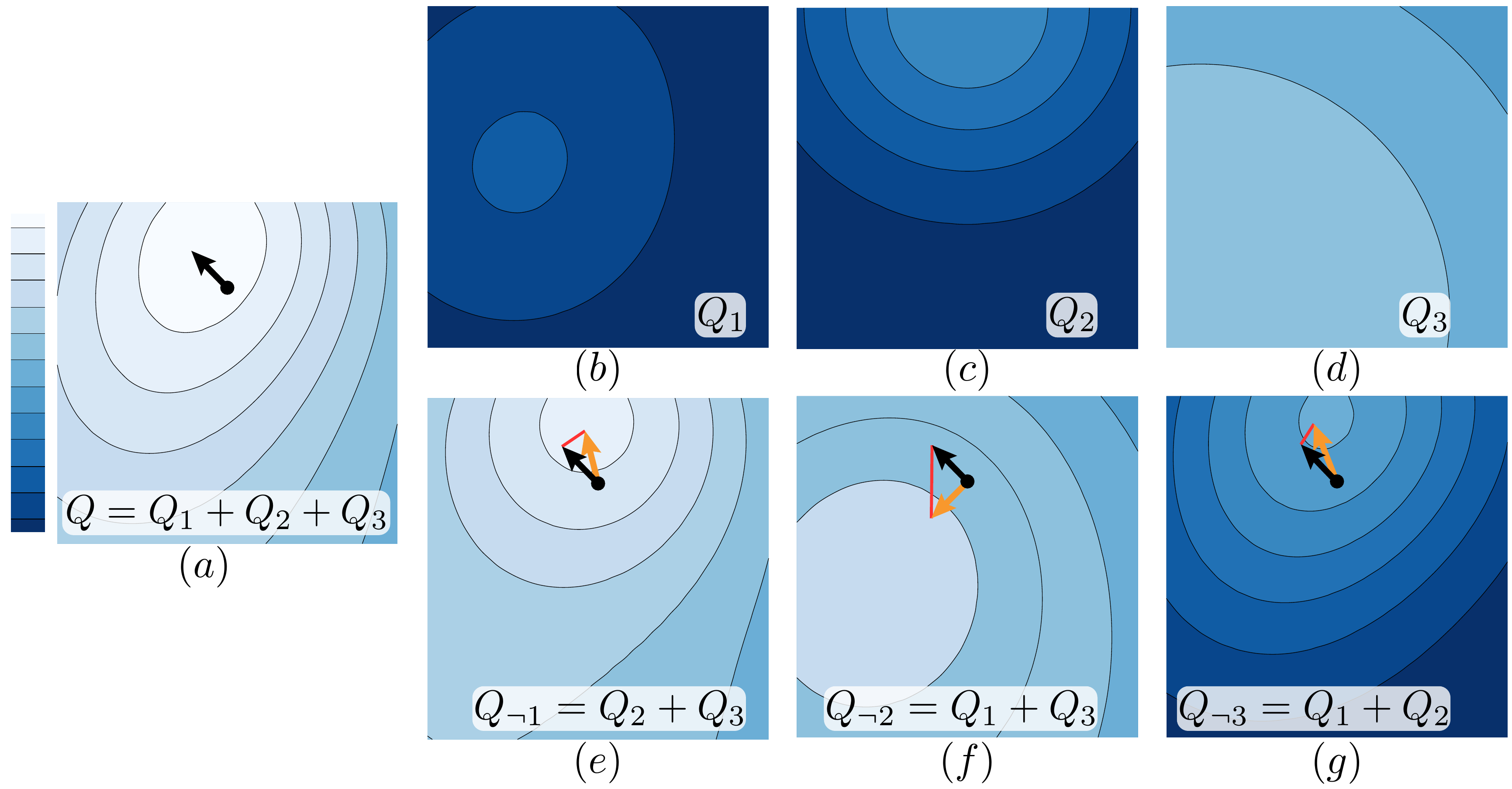}
\caption{\textbf{Influence calculation for a three-component environment with a 2D action space.} In this illustrative example, component $2$ has high influence, while components $1$ and $3$ have low influence, but for different reasons. Each figure shows a contour-map view of the Q-function over the 2D action space for the same state. \textbf{(a)} The composite Q-function the policy optimizes. The black point is a (near-optimal) policy selection, and the black arrow is the policy gradient. \textbf{(b)}-\textbf{(d)} Each component Q-function.  
\textbf{(e)}-\textbf{(g)} Component-ablated composite Q-function surfaces $Q_{\neg i}$ used to compute influence for each component. Influence is defined by the length of the red line, which is the difference of the policy gradients computed for the composite Q-function $Q$ (black arrow) and for $Q_{\neg i}$ (orange arrow).
\textbf{(e)} Component 1 has low influence because it is out-weighed by component 2. \textbf{(f)} Conversely, component 2 has high influence; without it, the policy would move toward the peak of component 1. \textbf{(g)} While component 3 contributes a large value, it has low influence because the value is nearly uniform; much of its value is received regardless of the action selected.}
\label{fig:influence_explainer}
\end{figure}
\section{Reward component influence}
\label{sec:influence}
 It can be difficult to understand how an agent's predictions interact to affect decision-making. We now introduce the reward \textit{influence} metric, which indicates how much each component contributes to an agent's decision. Intuitively, low influence means that removing a component would not alter decision-making; high influence means that removing it would significantly alter decision-making.

For multi-dimensional continuous actions, we define the \textit{optimal influence}
of component $i$ in state $s$ by how much the optimal policy $\pi^*$ in state $s$ differs from the optimal policy when component $i$ is removed: 
$\mathcal{I}^*_i(s) \triangleq ||\pi^*(s) - \pi_{\neg i}^*(s)||_2$, 
where 
$\pi_{\neg i}^* \triangleq \argmax_{\pi} E \left[ \sum_t^\infty \gamma^t \sum_{j\neq i} w_j R_j(s_t, a_t) | \pi \right]$.\footnote{Discrete-action spaces could use a probability distance measure, but we focus on continuous-action spaces.}

In practice, we apply two approximations to $\mathcal{I}^*_i(s)$ for computational efficiency. 
First, rather than compare the difference of optimal policies, we compare the difference between one step of policy improvement:
$\mathcal{I}_i^\pi(s) \triangleq ||\argmax_a Q^\pi(s, a) - \argmax_a Q_{\neg i}^\pi(s, a)||_2$,
where 
$Q_{\neg i}^\pi(s, a) \triangleq \sum_{j\neq i} w_j Q^\pi_j(s, a)$.  Second, since $\argmax$ is computationally demanding and sensitive to statistical noise, 
we replace the $\argmax_a Q(s, a)$ policy improvement operator with a policy gradient-step operator (typical in RL algorithms like SAC): $\bar{a} + \lambda \nabla_{\bar{a}}Q(s, \bar{a})$, where $\bar{a}$ is a deterministic policy action selection (such as the mode)
and $\lambda \in (0, 1)$ is a step size. 

When taking the difference of the gradient-step operator applied to the $Q^\pi$ and $Q^\pi_{\neg i}$ surfaces, the $\bar{a}$ terms cancel, and $\lambda$ can be factored out of the norm. The result is the \textit{influence} metric (\figref{fig:influence_explainer}):
\begin{equation}
 I^\pi_i(s; \theta) \triangleq \lambda ||\nabla_{\bar{a}} Q^\pi(s, \bar{a}; \theta) - \nabla_{\bar{a}} Q_{\neg i}^\pi(s, \bar{a}; \theta)||_2.
\end{equation}
The raw magnitudes of component influence can be informative by themselves; for example, a sharper Q-function surface leads to larger influence values.
However, to compare influence values across components, we typically compute the \textit{fractional influence} by normalizing the (always non-negative) influence: $\hat{I}_i^\pi(s; \theta) \triangleq \frac{I_i^\pi(s; \theta)}{\sum_j^m I_j^\pi(s; \theta)}$.

We use several techniques to visualize the fractional influence. 
For trajectories, plotting 
influences over timesteps may help identify and explain key decision points (\figref{fig:bw_eval_influence}). When studying an agent's behavior across training, we maintain summary statistics of fractional influence.
We visualize mean fractional influence across all components as a stack plot, sorted so that the component with the maximum influence at the end of training is at the bottom. Figure ~\ref{fig:ll_influence} shows such a diagram for the Lunar Lander environment; we provide figures for all environments in App.~\ref{appendix:influence}.

\newcommand{\influencefigsize}{0.28\columnwidth}
\begin{figure}[t]
    \centering
    \subfigure{
        \includegraphics[width=.95\linewidth]{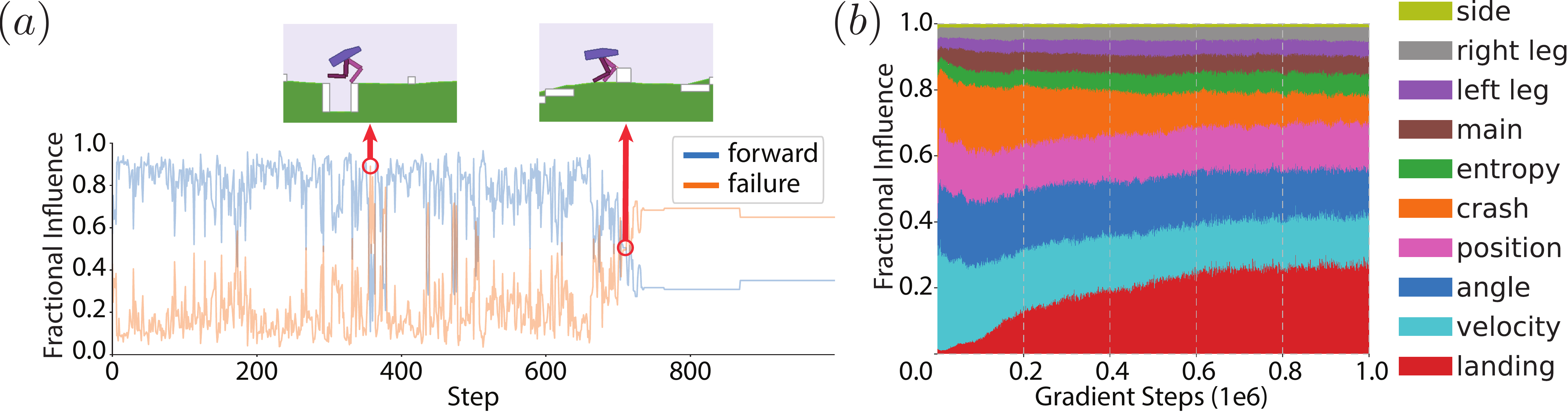}
        \label{fig:bw_eval_influence}
        
    }
    \subfigure{\label{fig:ll_influence}}
    \caption{\textbf{Fractional influence}  \textbf{(a)} Fractional influence for $1$ evaluation trajectory of BWH. 
    For this agent and trajectory, the \textit{forward} component dominates the \textit{failure} component in decision-making, except at two points (red arrows).
    At the first point, the agent was at risk of falling into a pit.
    After the agent stabilizes, the \textit{forward} component returns to being dominant.
    After the second point, the agent becomes unable to move forward. 
    \textit{Failure}'s dominance indicates that if the agent cannot move forward, it prefers standing to falling. If it is possible to surmount this obstacle, more exploration may be required.
    \textbf{(b)} Mean stacked fractional influence across 10 trials of Lunar Lander.
    We observe: (1) the sparse landing reward doesn't meaningfully contribute to decision-making in the earliest part of training; 
    (2) crashing influence starts high and then decreases, suggesting the agent initially learns to avoid crashing; 
    (3) the position, angle, and velocity reward components -- each a shaping reward -- maintain large influence, suggesting possible overreliance on shaping rewards;
    (4) numerous reward components never have much influence; they might be unnecessary or require more attention.}
\end{figure}
\section{Value decomposition: strategies for agent design}
\label{sec:analysis_examples}
Agent design is often a brute-force process of trial and error: when an agent doesn't perform as expected, we choose some aspect of the agent's design to vary, and then we train again. Although this approach can succeed, it can be expensive in both time and computation.

In this section, we illustrate a different approach, showing how a decomposed reward helps break from trial-and-error by encouraging the designer to consider an agent's point of view. In three examples, each using either the LL or BWH environments (Sec.~\ref{sec:sample_env}), we draw on value decomposition tools to: (1) identify learning problems by comparing components' empirical returns to their predictions; (2) constrain component value estimation; (3) identify adverse reward interactions with the influence metric; (4) dynamically re-weight reward components.

These examples are not just meant to demonstrate these specific techniques, nor to demonstrate significant performance improvements on these well-tested benchmark environments (in fact, SAC-D eventually produces good policies for both environments without additional tuning). Rather, the goal is to showcase a general approach to agent design for novel applications in which iteration and failure is costly.
Combined with statistical tools that allow us to reason with small-sample statistics, value decomposition provides a vocabulary to describe an agent's behavior and interpretable tools for targeted interventions. Even these small examples are more intuitive -- and less computationally demanding -- than they would be with a trial-and-error approach. 

For all examples, training parameters are identical to the robustness results of Sec.~\ref{sec:robustness}. 

\newcommand{\llfigsize}{0.35\columnwidth}
\begin{figure}[t]
    \centering
    \subfigure{
        \includegraphics[width=.75\textwidth]{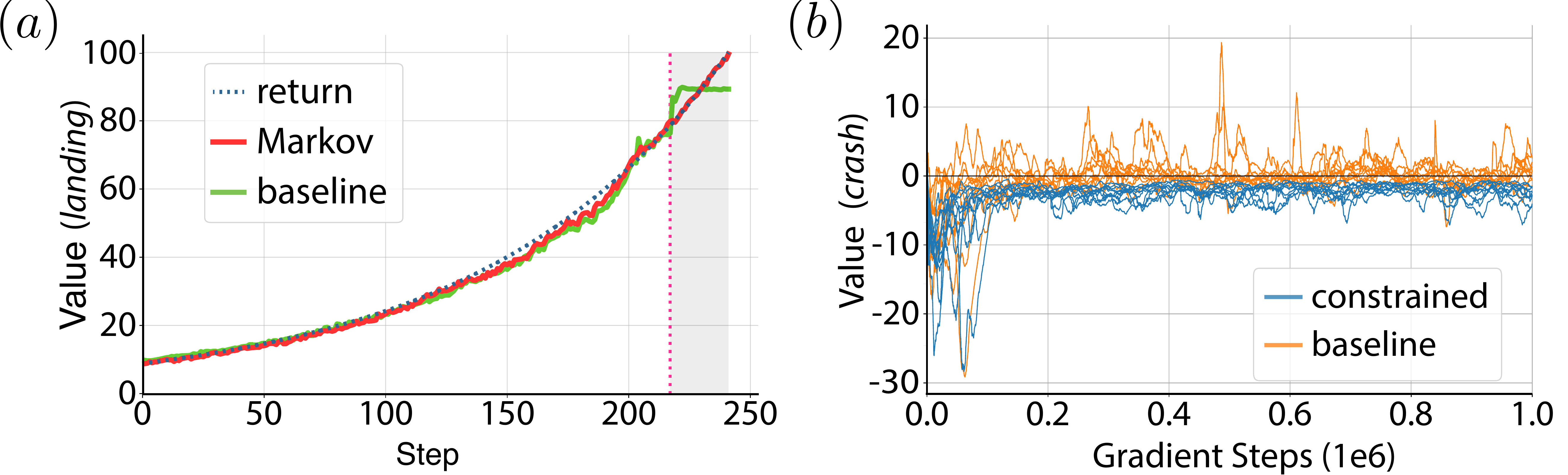}
        \label{fig:ll_predictions}
    }
    \subfigure{
        \label{fig:ll_constraints}
    }
    \caption{\textbf{Remedies for incorrect value predictions.}
    \textbf{(a)} For the \textit{landing} component empirical return (dashed blue), the value prediction (green) for a value function trained with baseline features matches the return well except for a final plateau in the shaded region. Training a value function with the added $V_0^\text{trace}$ feature (red) improves predictions in this region. 
    The dashed pink line indicates the point after which we compute error metrics in App.~\ref{appendix:markov}. \textbf{(b)} Value predictions for the \textit{crash} component with sign constraints (blue lines; using both target and prediction constraints) and without (orange lines). Each line represents an independent trial. Unconstrained estimates can be positive; constrained value estimates keep predictions negative.}
\end{figure}
\textbf{Diagnosing and improving insufficient features:}
We analyze the behavior of an agent trained on Lunar Lander by comparing each component's empirical return to the agent's value predictions, $Q_i$. The agent is trained to land successfully, and generally the component value predictions match their empirical returns well. Curiously, all component predictions are flat near the end of the episode. 
For most components, these flat predictions are a good match for their returns, but not for \textit{landing}.
Investigating the landing dynamics, we found the simulator waits many steps after touchdown before producing a landing reward. During this period, the observations are constant, suggesting the features are inadequate to represent \textit{landing}'s return.
To make the observations Markov, we introduce a new feature to the agent's observations that indicates the duration since the agent's velocity (horizontal, vertical and angular) went to zero: $V_{0}^\text{trace}(t) = V_{0}^{\text{steps}}(t)/c$, where $V_{0}^{\text{steps}}(t)$ is the number of time steps since all the velocities dropped below a threshold and $c$ is a fixed normalizing constant. With this feature, post-landing predictions show a marked improvement (see Figure~\ref{fig:ll_predictions} and Appendices~\ref{appendix:detail},~\ref{appendix:markov}).

\textbf{Diagnosing and mitigating value errors using domain knowledge:}
Lunar Lander's design makes clear that certain reward components are always non-positive (\textit{crash, main, side}) while others are always non-negative (\textit{landing}).
However, we observe that the agent's decomposed value predictions do not always match these bounds. In particular, value predictions of \textit{crash} have a tendency to oscillate about $0$ after the agent learns to land. 
Value decomposition allows us to explicitly enforce a sign constraint on \textit{crash} (Figure~\ref{fig:ll_constraints}; see App.~\ref{app:constraint} for details). In this particular example, constraints do not alter policy learning performance, but the resulting predictions are easier to interpret, and the same technique may improve performance in more complex environments.

\textbf{Mitigating an adverse reward with component weight scheduling:}
Under the BWH reward function, a random policy is far more likely to experience an unsuccessful outcome (falling over) than a successful one (walking forward). This bias can inhibit agent exploration early in training, causing an agent's policy to fall into a local minimum (the agent stands still).
Here, we diagnose this dynamic using component predictions and influence metrics, and remedy it by varying a single component weight during training.

We find that the \textit{failure} component's fractional influence dominates the \textit{forward} component's fractional influence early in training, and that this relationship reverses as agent performance improves (\figref{fig:bwh_influence}).
The \textit{forward} component's near-zero value predictions (\figref{fig:bwh_value_predictions}) early in learning indicate that in many episodes, the agent neither moves nor expects to move (\figref{fig:bwh_score_dist}). Early dominance by the easy-to-observe \textit{failure} suggests that it is inhibiting exploration. 

To mitigate this problem, we 
vary the \textit{failure} component weight, $w_\text{failure}$, from $0.01$ to $1$ over the course of learning according to the schedule described in App.~\ref{appendix:annealing}.
This schedule significantly increases the agents' \textit{forward} progress (\figref{fig:bwh_score_dist}),
and accelerates learning.

\section{Related work}
\label{sec:related}
Our work builds upon earlier studies of value decomposition in the explainable reinforcement learning (XRL) literature: DrQ~\cite{Juozapaitis2019ExplainableRL}, DrSARSA~\cite{Russell2003QDecompositionFR}, and HRA~\cite{HybridRA}.\footnote{For a broad overview of XRL, we direct the reader to Heuillet et al.~\cite{heuillet2021explainability}.} Like our approach, these methods learn separate value function estimates for each term of a linear reward function. DrQ and HRA are off-policy Q-learning-like methods, while DrSARSA is an on-policy method. HRA does not converge to a globally optimal policy, as each value function is only {\em locally}-optimal for the reward component it measures. The RDX and MSX metrics proposed by DrQ could be adopted in our setting, but the influence metric is easier to use with continuous-actions, and aggregate with summary statistics. Our approach improves upon these prior contributions by: (1) working in continuous-action and discrete-action environments; (2) allowing for and demonstrating dynamic re-weighting of reward components during training; and (3) being applicable to a family of actor-critic methods. While these approaches and ours explore using value decomposition for explainability, these works focus on describing why an agent took certain actions to users, whereas we focus on how to use value decomposition to diagnose and remedy learning challenges.

The Horde architecture~\cite{sutton2011horde} and UVFA~\cite{Schaul2015UniversalVF}, methods for multi-goal learning, also employ multiple value functions (one for each goal). UVFAs use a parameterized continuous space of goals, while Horde makes multiple discrete value function predictions. The value functions in our work are conditioned on a policy that optimizes the global reward, whereas in Horde and UVFAs the value functions are conditioned on independent policies that greedily optimize local goals (similar to HRA). Additionally, in our approach, the composite value function can be recovered from the components. Other value decomposition work includes Empathic Q-learning~\cite{laroche2017multi} and Orchestrated Value Mapping~\cite{fatemi2022orchestrated}. The primary difference between our work these other approaches regards the motivation and application of value decomposition. These works focus on how value decomposition can improve sample efficiency, generalization, or other aspects of the core learning problem, rather than how to diagnose and remedy problems.

Mathematically, value decomposition bears resemblance to work on successor features, which has focused primarily on state representation and transfer learning \cite{barreto2017successor,Borsa2019UniversalSF,Barreto2019TheOK,Grimm2019LearningIR}. 
Methods for multi-objective RL~\cite{Roijers,hayes_practical_2022} learn sets of policies, each with a distinct linear weight over multiple reward objectives. Our work also recovers the value function for different linear reward preferences, but uses this capability to diagnose behavior and learning problems rather than to learn multiple policies.

\begin{figure}[tp]
    \centering
    \subfigure{\includegraphics[width=0.95\columnwidth]{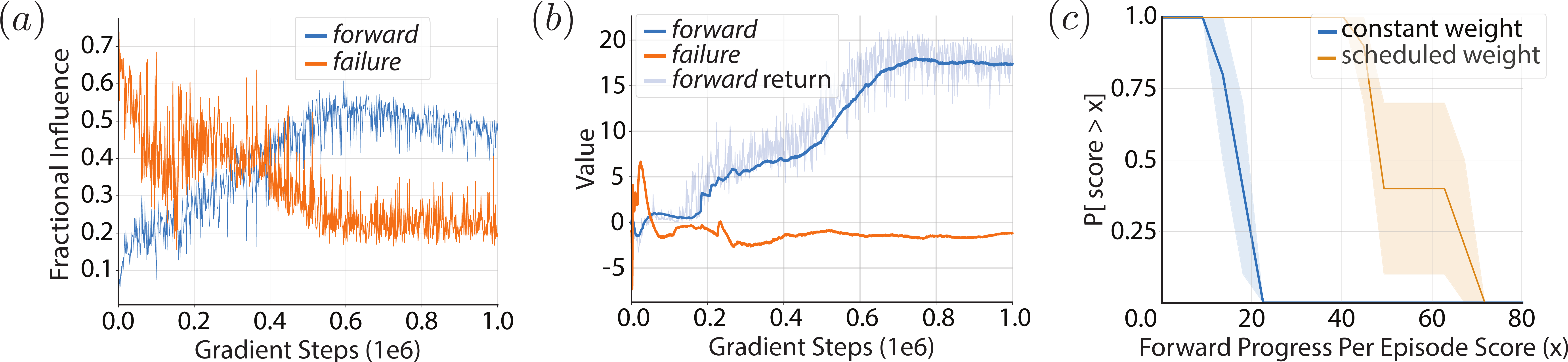}\label{fig:bwh_influence}}
    \subfigure{\label{fig:bwh_value_predictions}}
    \subfigure{\label{fig:bwh_score_dist}}
    \caption{\textbf{Remedy for an adverse reward}
    \textbf{(a)} 
    The initially low fractional influence of \textit{forward} (blue curve) indicates it may be initially inhibited by failure penalties with higher influence.
    \textbf{(b)} The \textit{forward} value predictions and returns (blue curves) are near zero early in training, indicating the agent does not expect to move forward. \textbf{(c)} An agent trained with its failure weight scheduled from a value close to zero up to one (orange curve) improves early forward progress compared to an agent with a constant failure weight (blue curve). Forward progress is measured as the average forward progress return from all evaluation episodes of the first 200k gradient steps. Shaded areas represent $95\%$ bootstrap confidence intervals from 10 trials.
    }
    \label{fig:bwh_grid}
\end{figure}

\section{Concluding remarks}
\label{sec:conclusion}

We have argued that the iterative design of reinforcement learning agents can be improved through the use of \textit{value decomposition}, in which we keep individual reward components separate and learn value estimates of each. 
We provided a simple prescription for deriving value decomposition algorithms from actor-critic methods, and applied it to SAC to derive the SAC-D algorithm. Combined with the CAGrad method, SAC-D meets or exceeds SAC's performance in all environments we tested.
We introduced the \textit{influence} metric, and demonstrated its use in measuring each reward component's effect on an agent's decisions. 
Finally, we provided several examples of how value decomposition can diagnose and remedy agent learning problems.

Although value decomposition is a simple and broadly applicable tool,
we note the following limitations. (1) Our method requires a composite reward function of multiple components. (2) We only study linear reward decomposition. (3) Component predictions only tell you the agent expectations under the single learned policy; changing the weights doesn't tell you what to expect after re-optimizing the policy for them. 
(4) Our approach benefits most with methods that learn Q-functions. Methods that optimize the policy with empirical returns have a weaker link between agent expectations (component Q-functions) and policy decisions. The influence metric also requires a Q-function model.

Our method presents the same societal benefits and risks as other RL methods. However, we believe this technology has a net beneficial impact, because making agent decisions more introspectable enables developers to catch problematic behavior before deploying the technology. One particular concern, however, is that value estimates represent an agent's beliefs, not ground truth; this should be kept in mind when such predictions are used in real-world decision-making, as they may reinforce biases or lead to incorrect conclusions.

\bibliographystyle{abbrvnat}
\bibliography{main}

\newpage
\appendix
\section*{Appendix}

\section{Value decomposition of linear reward functions}
The following proof shows that composite \qfunc{} can be recovered from the component \qfuncs{} of a linear reward function.
\begin{theorem}
\label{thm:valdecomp}
Let the reward function $R(s, a)$ of an MDP be a linear combination of a finite number of $\ncomponents$ components: 
$$
R(s, a) \triangleq \sum_i^\ncomponents w_i R_i(s, a),
$$ 
where $w_i \in \mathbb{R}$ is a weight for the $i$th component and $R_i(s, a) \rightarrow \mathbb{R}$ defines the $i$th reward component for state-action pair $s$ and $a$. 

Let $Q^\pi_i$ be the component \qfunc{} for the $i$th reward component under policy $\pi$: 
$$
Q^\pi_i(s, a) \triangleq E\left[\sum_{t=0}^\infty \gamma^t R_i(s_t, a_t)| s_0 = s, a_0 = a\right].
$$

Then the composite \qfunc{} 
\begin{align*}
Q^\pi(s, a) \triangleq & E_\pi\left[\sum_{t=0}^\infty \gamma^t R(s_t, a_t) | s_0 = s, a_0 = a\right] \\
=&  \sum_i w_i Q_i^\pi(s, a).
\end{align*}
\end{theorem}
\begin{proof}
In the following proof, we use the annotation ``By def.'' to mean substitution of a definition; ``Distrib.'' to mean by the distributive property; ``C.A.'' to mean by the commutativity and associativity property; and ``Lin. E.'' to mean by the linearity of expectation property.
For notational simplicity, the expectation $E_\pi\left[ ... \right]$ implies the conditional expectation $E_\pi\left[ ... \right | s_0 = s, a_0 = a]$.
\begin{align*}
Q^\pi(s, a) &\triangleq E_\pi\left[\sum_{t=0}^{\infty}\gamma^t R(s_t, a_t) \right] & \\
&= E_\pi\left[\sum_{t=0}^{\infty}\gamma^t \sum_i^k w_i R_i(s_t, a_t) \right] &  \text{By def.} \\
&= E_\pi\left[\sum_{t=0}^{\infty} \sum_i^k \gamma^t w_i R_i(s_t, a_t) \right] &  \text{Distrib.} \\
&= E_\pi\left[\sum_i^k \sum_{t=0}^{\infty} \gamma^t w_i R_i(s_t, a_t) \right] &  \text{C.A.} \\
&= E_\pi\left[\sum_i^k w_i \sum_{t=0}^{\infty} \gamma^t R_i(s_t, a_t) \right] &  \text{Distrib.} \\
&= \sum_i^k E_\pi\left[ w_i \sum_{t=0}^{\infty} \gamma^t R_i(s_t, a_t) \right] &  \text{Lin. E.} \\
&= \sum_i^k w_i E_\pi\left[ \sum_{t=0}^{\infty} \gamma^t R_i(s_t, a_t) \right] &  \text{Lin. E.} \\
&= \sum_i^k w_i Q_i^\pi(s, a) &  \text{By def.}\\
\end{align*}
\end{proof}

\section{Environment reward components}
\label{appendix:environment}

The decomposed reward components used in the various environments are listed in this section. All of these are internally implemented in Gym~\cite{brockman2016openai}, but only the aggregated scalar rewards are exposed by default. To this end, we implemented a modified version to expose these decomposed reward functions. For Ant-v3 and Humanoid-v3, there is an issue with contact calculation due to the compatibility with the latest open-sourced version of MuJoCo simulator (\url{https://github.com/openai/gym/issues/1541}). Therefore, we removed the contact cost component in our experiments.

\textbf{LunarLander-v2.}
\begin{itemize}
    \item \texttt{main}: cost for using the main engine.
    \item \texttt{side}: cost for using the side engines.
    \item \texttt{crash}: penalty for crashing.
    \item \texttt{landing}: reward for landing.
    \item \texttt{left\_leg}: shaping reward for the left leg contacting the ground.
    \item \texttt{right\_leg}: shaping reward for the right leg contacting the ground.
    \item \texttt{angle}: shaping penalty for being oriented away from the vertical.
    \item \texttt{position}: shaping penalty to encourage the lander to come down and to the center of the landing pad.
    \item \texttt{velocity}: shaping penalty to drive the lander to prefer low speed.
\end{itemize}

\textbf{Hopper-v3.}
\begin{itemize}
    \item \texttt{forward}: forward progress along the movement axis.
    \item \texttt{control\_cost}: cost of actuator actions.
    \item \texttt{healthy}: constant alive bonus.
\end{itemize}

\textbf{BipedalWalker-v3.}
\begin{itemize}
    \item \texttt{forward}: forward progress.
    \item \texttt{head}: steadiness of the head.
    \item \texttt{control\_cost}: cost of actuator actions.
    \item \texttt{failure}: penalty at failure.
\end{itemize}

\textbf{BipedalWalkerHardcore-v3.}
\begin{itemize}
    \item \texttt{forward}: forward progress.
    \item \texttt{head}: steadiness of the head.
    \item \texttt{control\_cost}: cost of actuator actions.
    \item \texttt{failure}: penalty at failure.
\end{itemize}

\textbf{HalfCheetah-v3.}
\begin{itemize}
    \item \texttt{forward}: forward progress along the movement axis.
    \item \texttt{healthy}: constant alive bonus.
\end{itemize}

\textbf{Walker2d-v3.}
\begin{itemize}
    \item \texttt{forward}: forward progress along the movement axis.
    \item \texttt{control\_cost}: cost of actuator actions.
    \item \texttt{healthy}: constant alive bonus.
\end{itemize}

\begin{samepage}
\textbf{Ant-v3.}
\begin{itemize}
    \item \texttt{forward}: forward progress along the movement axis.
    \item \texttt{control\_cost}: cost of actuator actions.
    \item \texttt{healthy}: constant alive bonus.
\end{itemize}
\end{samepage}

\textbf{Humanoid-v3.}
\begin{itemize}
    \item \texttt{forward}: forward progress along the movement axis.
    \item \texttt{control\_cost}: cost of actuator actions.
    \item \texttt{healthy}: constant alive bonus.
\end{itemize}

\section{Experiment details}
\label{appendix:detail}

All algorithms are trained in a distributed manner with separated processes of a rollout worker that collects experiences and a trainer that updates the neural network parameters asynchronously.
We used Reverb~\cite{reverb} to implement the experience replay buffer.

The policy network outputs the vector mean ($\mu$) and (diagonal) log standard deviation ($\log(\sigma)$) of a squashed Gaussian distribution whose output is constrained to $(-1, +1)$. 
During data collection actions are sampled by first sampling the random (unsquashed) vector $Z_t \sim N(\mu(S_t),\sigma(S_t))$, and then applying $\tanh$ to that to get the action: $A_t \gets \tanh(Z_t)$. The action was scaled to the appropriate domain before being executed in the environment for environments with different action domains than $(-1, +1)$.
The trained policies were evaluated for 10 episodes every 1000 gradient steps. During an evaluation episode, actions were drawn deterministically from the policy by using the output mean $Z_t \gets \mu(S_t)$, instead of sampling.

In SAC-D and SAC-D-Naive training, the gradients of the hidden layers are divided by the size of reward dimension. For SAC-D-CAGrad training, we implemented the CAGrad part based on the publicly available code\footnote{https://github.com/Cranial-XIX/CAGrad} released by the authors.
We found that the scipy's \texttt{minimize} function used for the inner optimization is sensitive to the gradient scale.
Therefore the gradients are normalized by an averaged norm during the inner optimization.
Table~\ref{tab:hyperparams} shows the hyperparameters used in the experiments.

\begin{table}[ht]
\centering
\caption{Hyperparameters used in robustness experiments}
\begin{tabular}{@{}llr@{}} \toprule
    Algorithm            & Hyperparameter & Value \\
    \midrule 
    SAC & Actor learning rate   & 3e-4 \\
                         & Critic learning rate  & 3e-4 \\
                         & Discount factor       & 0.99 \\
                         & Mini-batch size       & 256 \\
                         & Entropy target        & $-|A|$ \\
                         & Target update rate    & 5e-3 \\
                         & Hidden layers         & [256, 256] \\
                         & Activation            & ReLU \\
                         & Replay buffer size    & 1000000 \\
                         & Optimizer             & Adam~\cite{kingma2014adam} \\
                         & Steps until timeout   & 1000 \\
    SAC-D-CAGrad         & $c$~\cite{cagrad}     & 0.5 \\
    \bottomrule
\end{tabular}
\label{tab:hyperparams}
\end{table}

\subsection{Computation}
\label{appendix:computation}

Our computational resources were a custom cluster of machines deployed to allow distributed training in which data collection and training are performed on different machines, the \textit{rollout worker} and \textit{trainer} respectively. The rollout worker used a single CPU with 2 GB of RAM (AWS c5.2xlarge: Intel Xeon, 3.9 GHz) while the trainer used 8 CPU cores with 8 GB of RAM (AWS p3.2xlarge: Intel Xeon E5-2686 v4 processors, 2.3 GHz).

Computing the mean training steps per second for Bipedal Walker Hardcore we see that SAC-D is moderately slower than SAC, but SAC-D-CAGrad is significantly slower (Table~\ref{tab:compute}).
Applying CAGrad slows down the rate of computation because on each step it solves an additional optimization step in order to compute the gradient. For the values reported in Table~\ref{tab:compute} we use the default minimization algorithm from SciPy \citep{2020SciPy-NMeth} for this computation, but other approaches might be used to improve the computational efficiency.

\begin{table}[ht]
    \centering
    \caption{Training Steps/s for different algorithms on BipedalWalkerHardcore-v3}
    \begin{tabular}{@{}lr@{}} \toprule
        Variant & Avg Training Steps/s (std) \\
        \midrule
        SAC &  174 (13) \\
        SAC-D & 142 (8) \\
        SAC-D-CAGrad & 26 (2)\\
        \bottomrule
    \end{tabular}
    \label{tab:compute}
\end{table}

\section{Training curves for robustness experiment}
\label{appendix:training_curves}
In~\figref{fig:robustness} we show the training curves for all environments for SAC, SAC-D-Naive, SAC-D, and SAC-D-CAGrad. In most cases, SAC, SAC-D, and SAC-D-CAGrad resulted in final policies with comparable performance, with SAC-D-Naive sometimes under performing the others. Of these, the results on Ant-v3 and BipedalWalkerHardcore-v3 are the most atypical. On Ant-v3, SAC-D performed significantly worse than the others. See App.~\ref{appendix:ant_analysis} for a larger investigation of this result. On BipedalWalkerHardcore-v3, we found that SAC-D-CAGrad significantly outperformed the other methods. Because we trained agents using an asynchronous training server, one possible reason for this result is that computational differences in the experiments resulted in different data distributions, rather than the improved performance being due to the differences in the algorithms. We trained SAC-D-CAGrad and SAC synchronously (with one gradient step per environment step) on BipedalWalkerHardcore-v3 to rule out this factor. While we found that asynchronous training tended to be more effective than synchronous training, SAC-D-CAGrad still outperformed SAC after 1 million gradient steps (see \figref{fig:sync_training}). These results suggest that value decomposition with CAGrad may in some cases benefit learning compared to baseline algorithms without value decomposition. We leave investigating this question further for future work.
\begin{figure}[h]
    \centering
    \subfigure[LunarLander-v2]{\includegraphics[width=0.3\columnwidth]{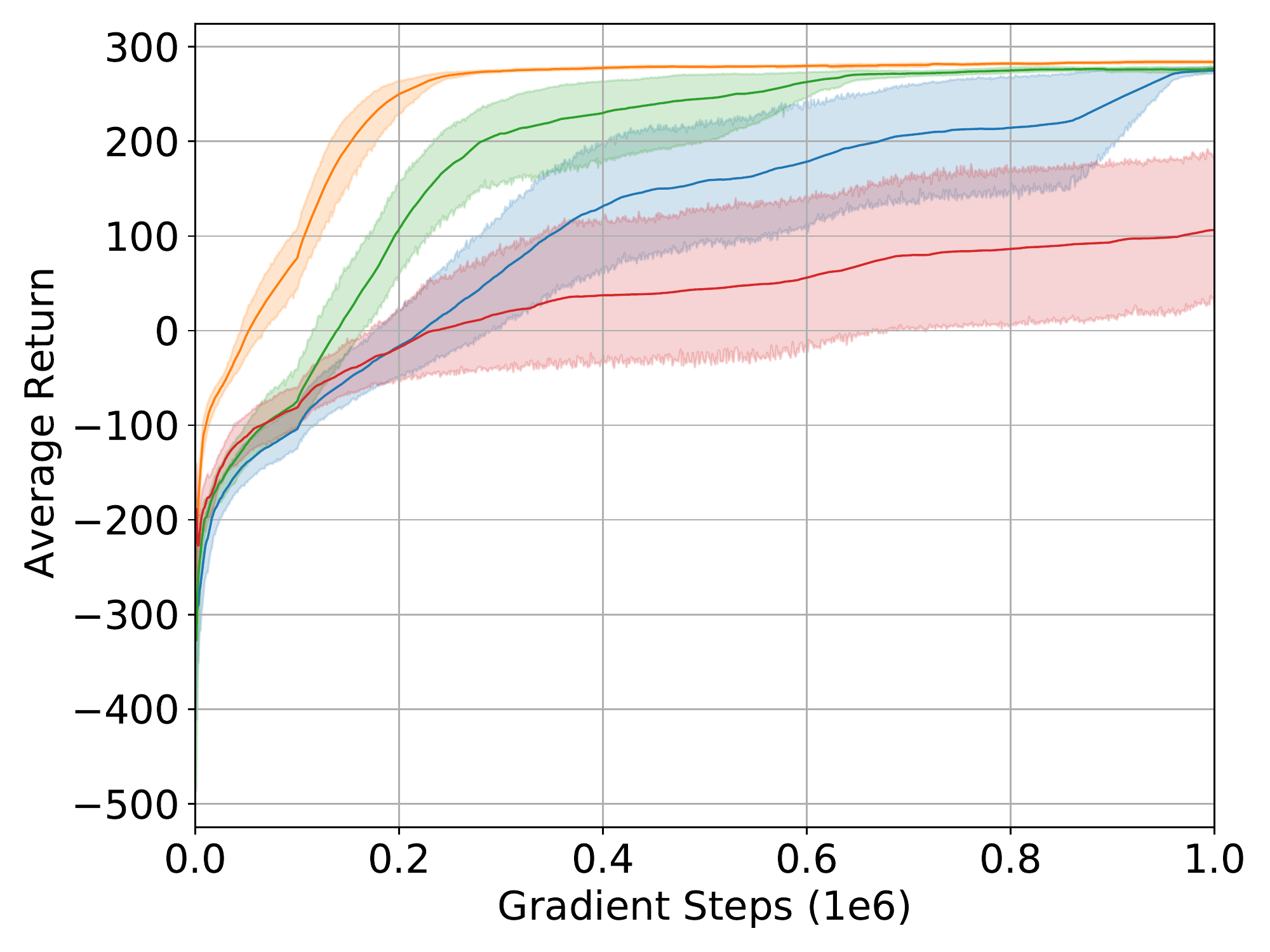}}
    \subfigure[BipedalWalker-v3]{\includegraphics[width=0.3\columnwidth]{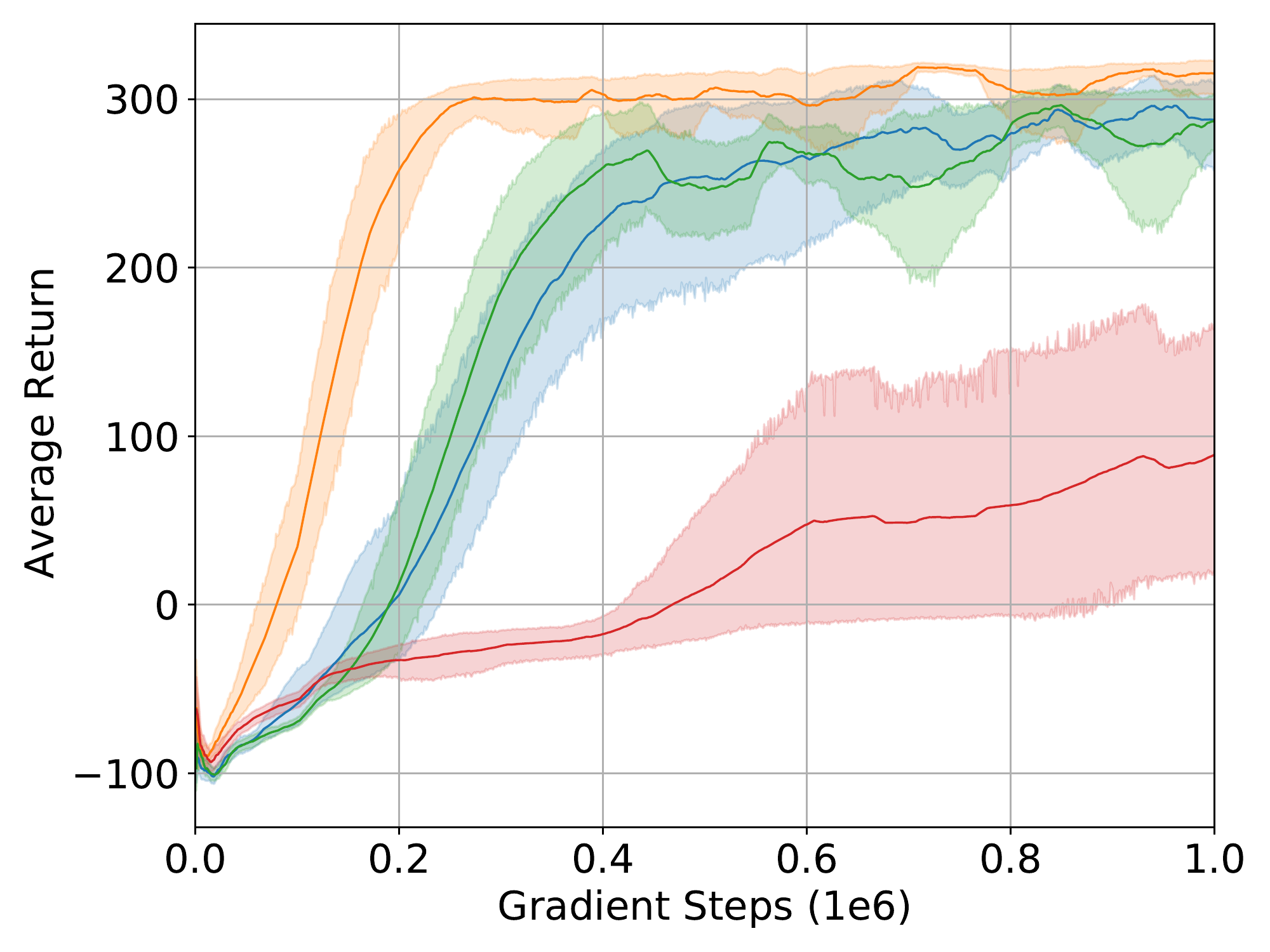}} \\
    \subfigure[BipedalWalkerHardcore-v3]{\includegraphics[width=0.3\columnwidth]{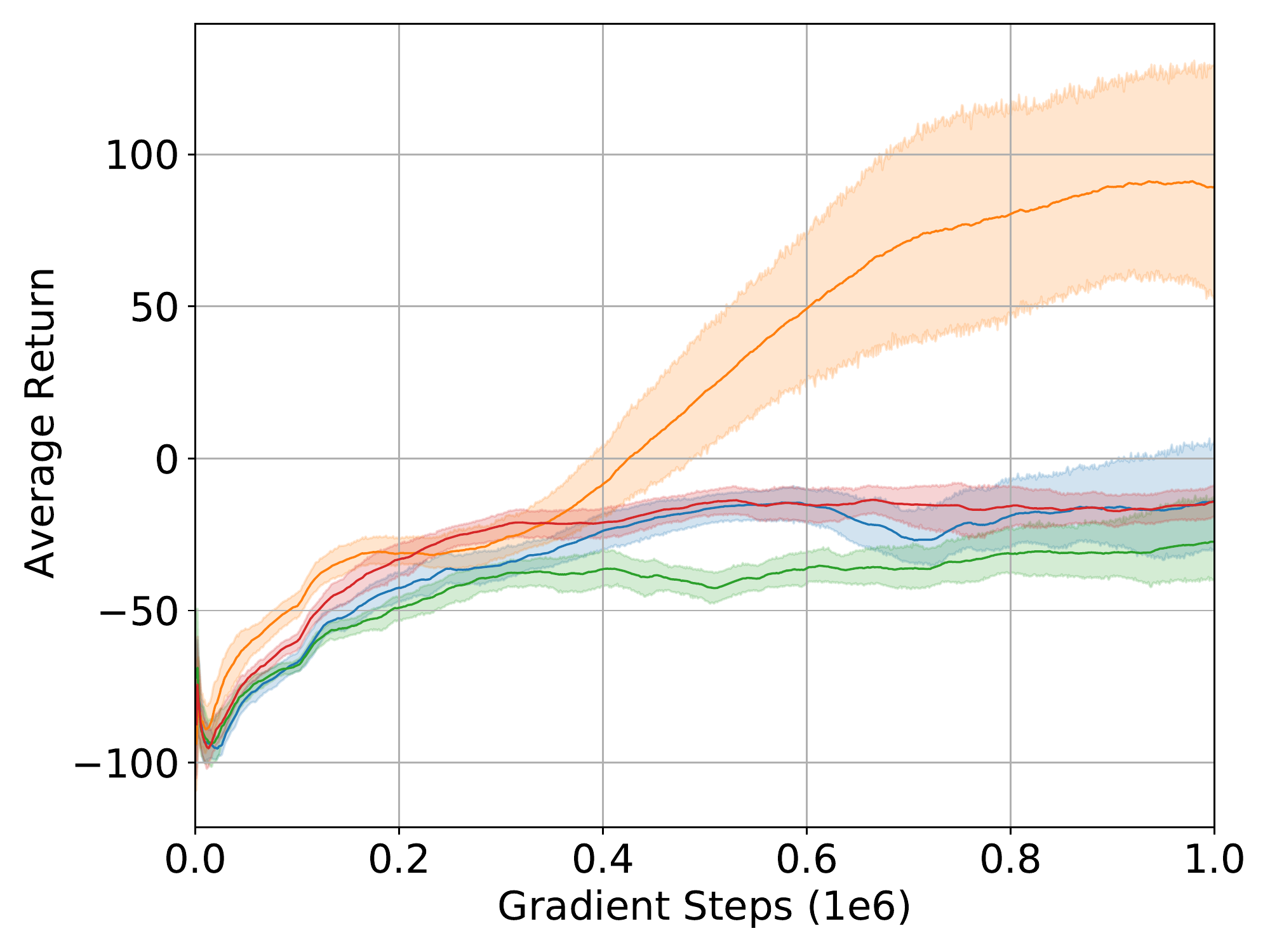}}
    \subfigure[Hopper-v3]{\includegraphics[width=0.3\columnwidth]{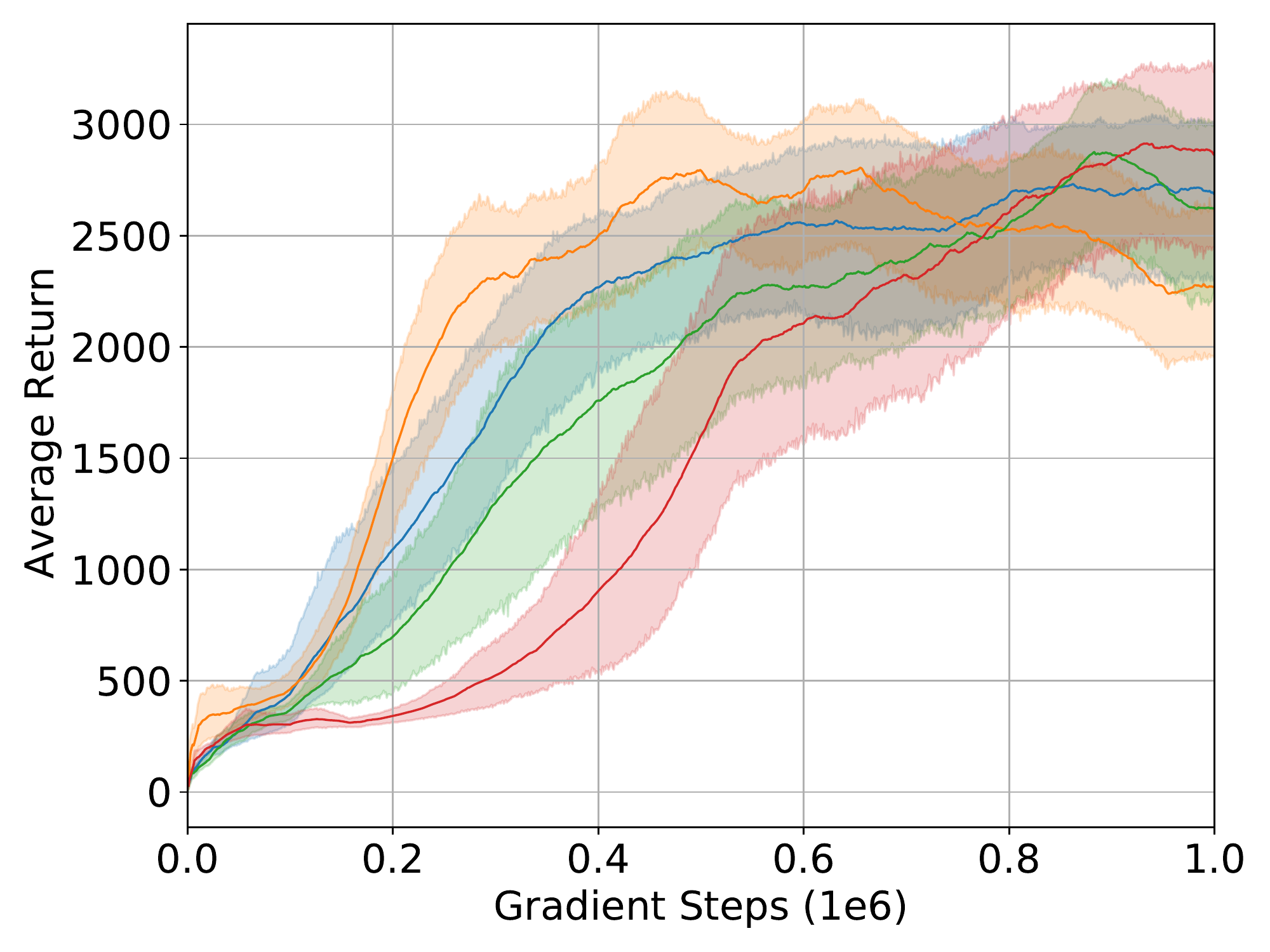}} \\
    \subfigure[HalfCheetah-v3]{\includegraphics[width=0.3\columnwidth]{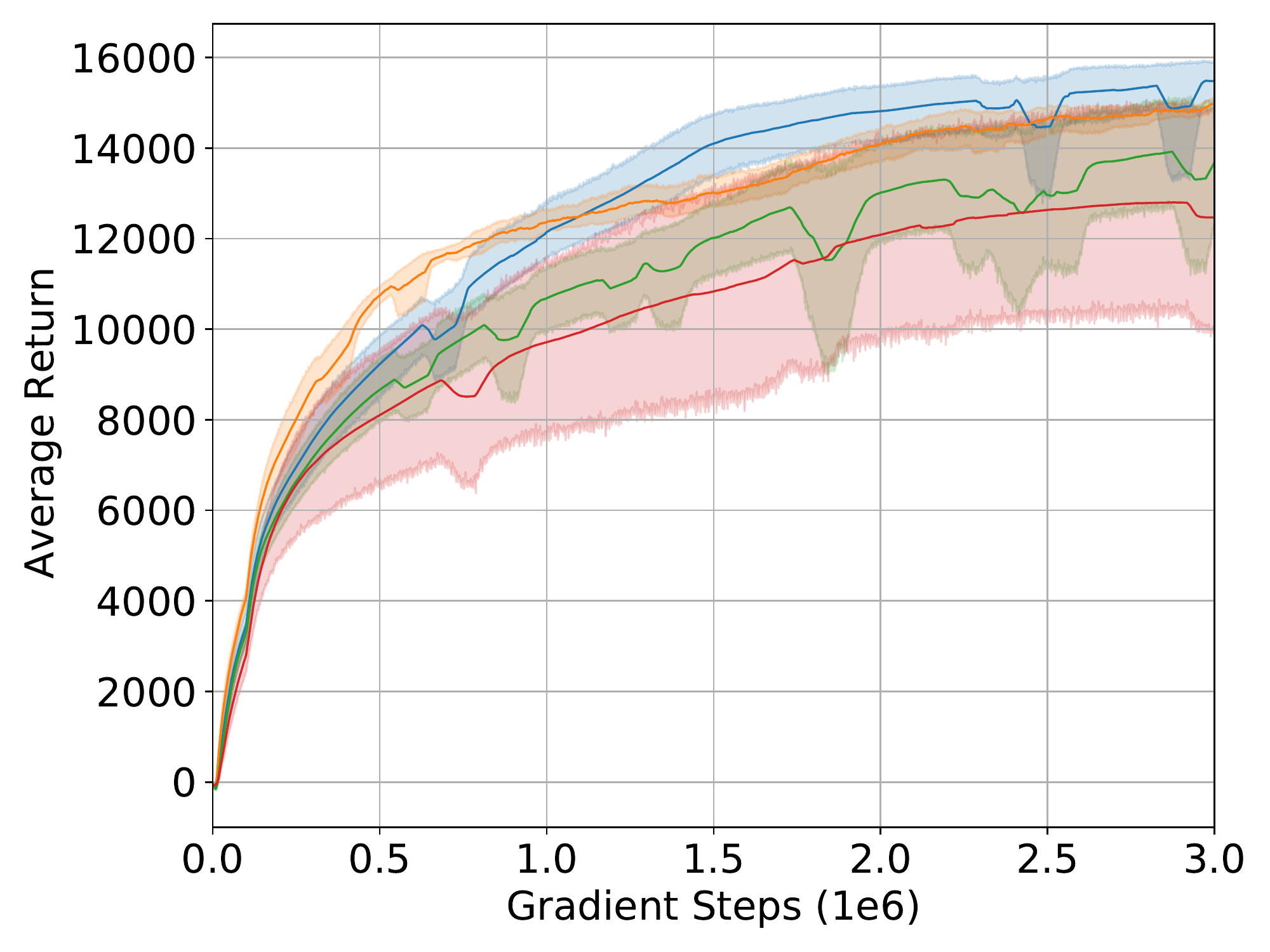}}
    \subfigure[Walker2d-v3]{\includegraphics[width=0.3\columnwidth]{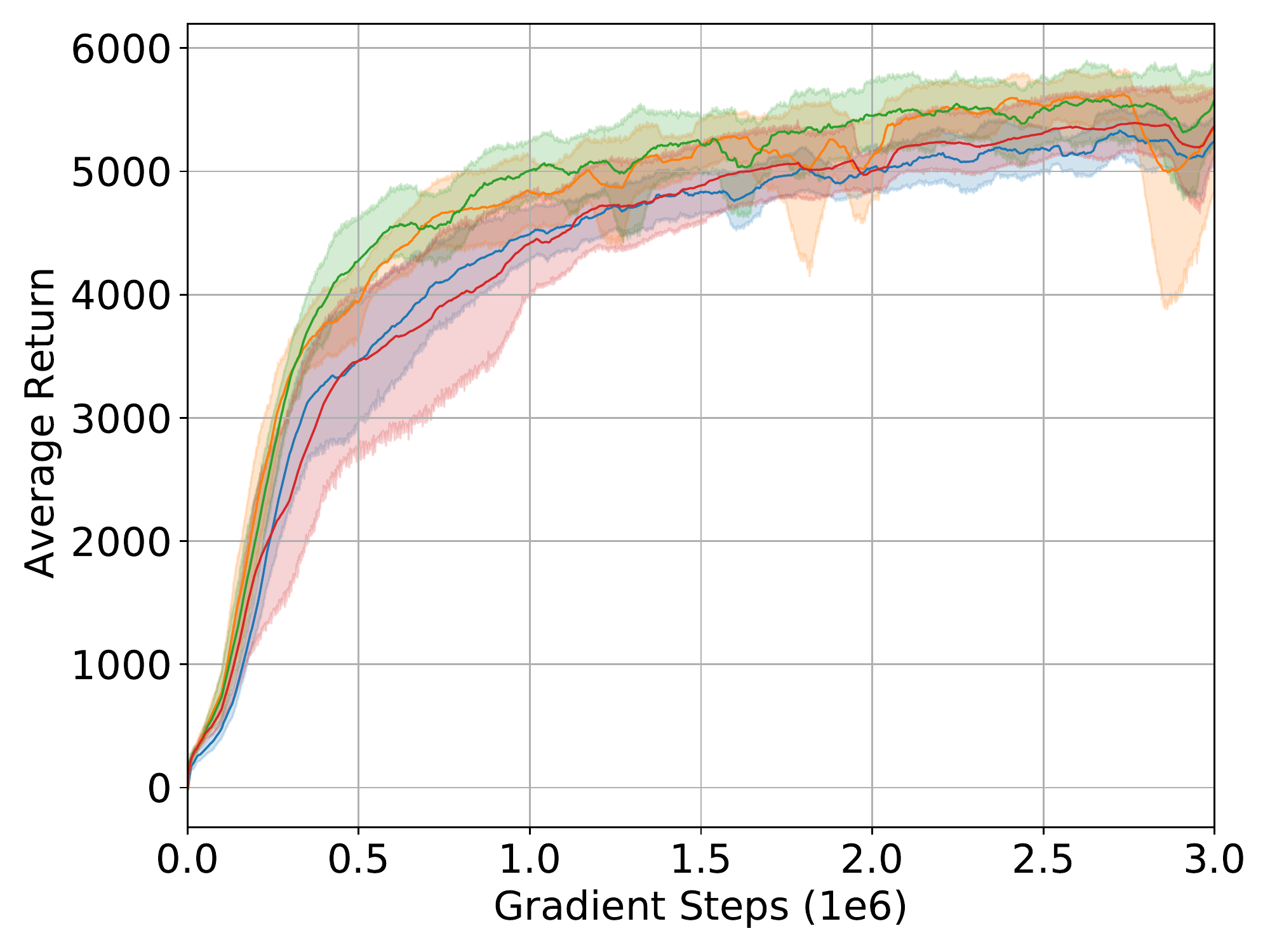}} \\
    \subfigure[Ant-v3]{\includegraphics[width=0.3\columnwidth]{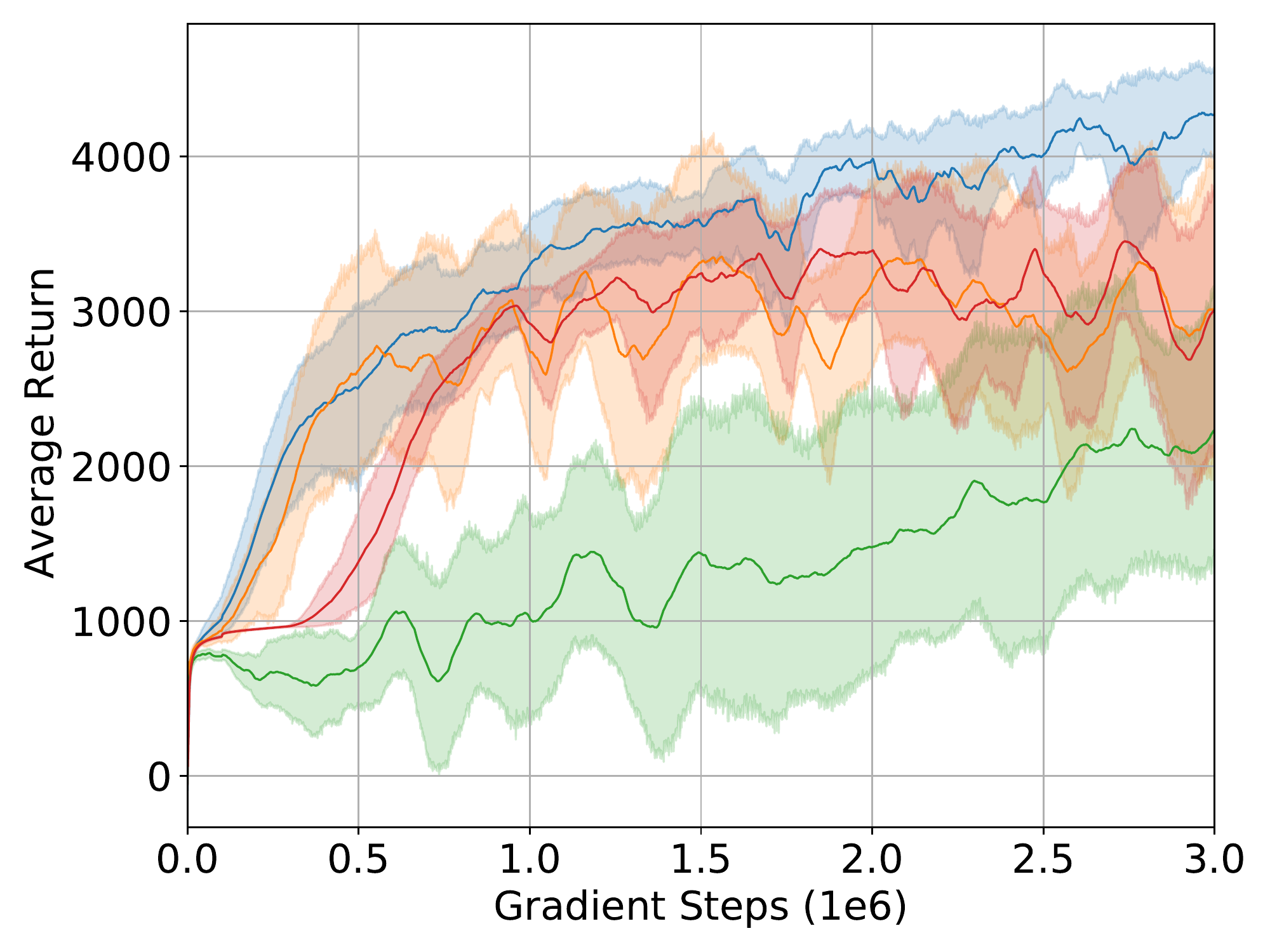}}
    \subfigure[Humanoid-v3]{\includegraphics[width=0.3\columnwidth]{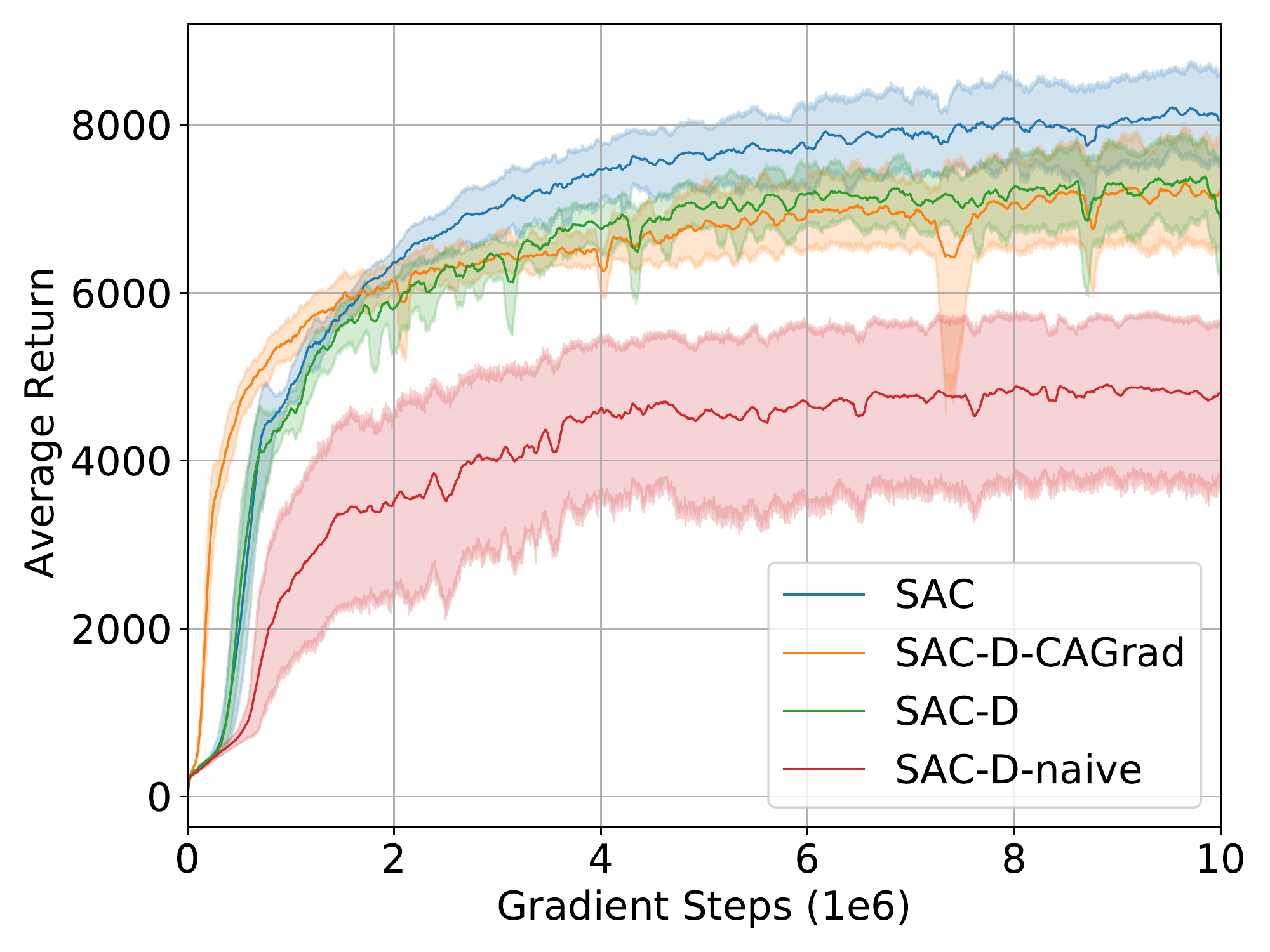}}
    \caption{\small Training curves on continuous-control benchmarks. The solid lines represent the average episode returns over 10 trials. The shaded region represents the confidence interval. The lines are uniformly smoothed for visual clarity.}
    \label{fig:robustness}
\end{figure}
\begin{figure}[h]
\centering
\includegraphics[width=0.3\columnwidth]{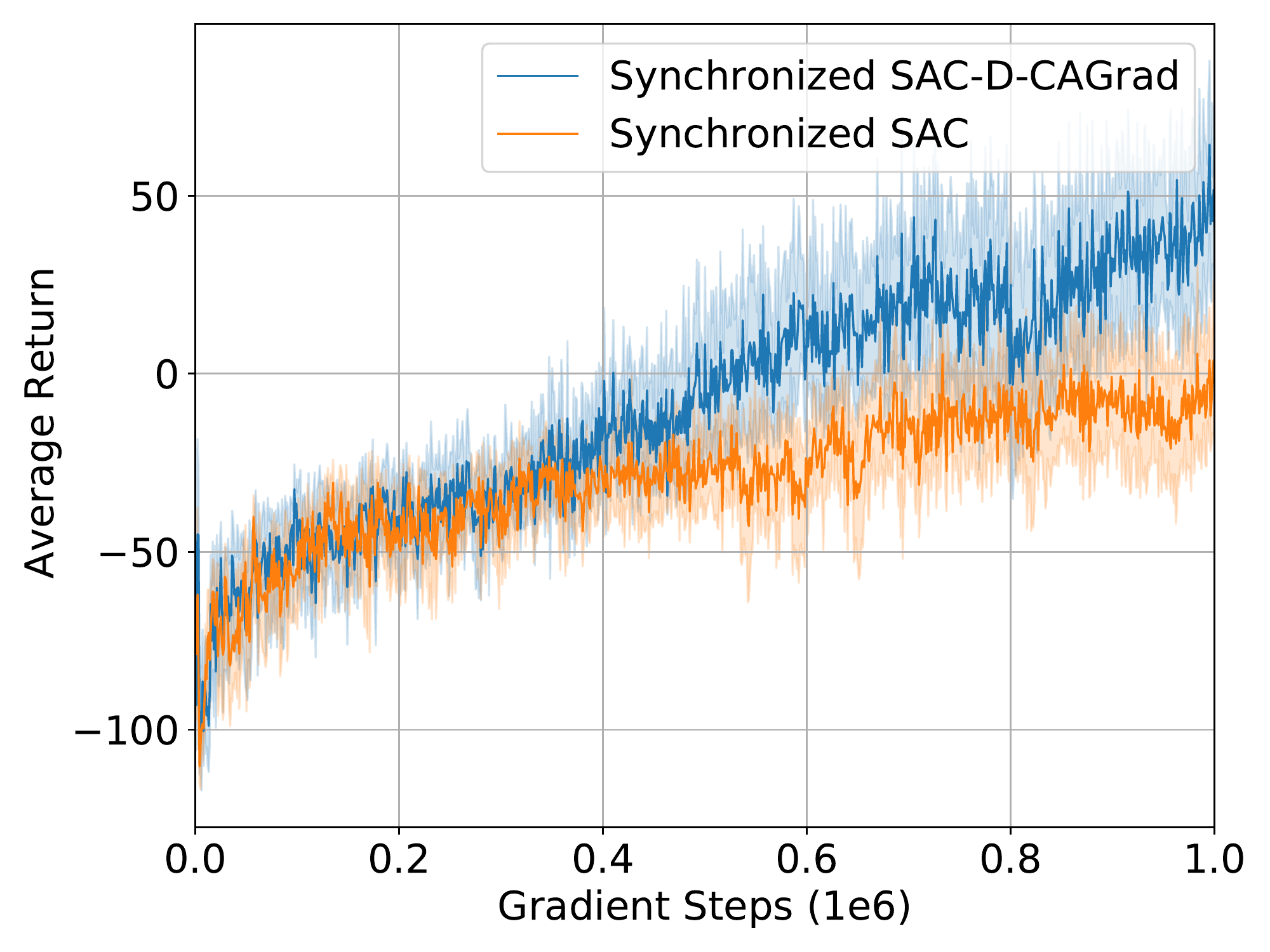}
\caption{\small Average training curves of synchronized SAC and SAC-D-CAGrad on BipedalWalkerHardcore-v3 over ten trials. SAC-D-CAGrad still outperforms SAC in the synchronized setting after 1 million gradient steps, suggesting there may sometimes be an algorithmic benefit to using SAC with value decomposition.}
\label{fig:sync_training}
\end{figure}

\section{Performance Analysis of SAC-D on Ant-v3}
\label{appendix:ant_analysis}

\begin{figure}[ht]
    \centering
    \subfigure[Frequency of termination]{\includegraphics[width=0.3\columnwidth]{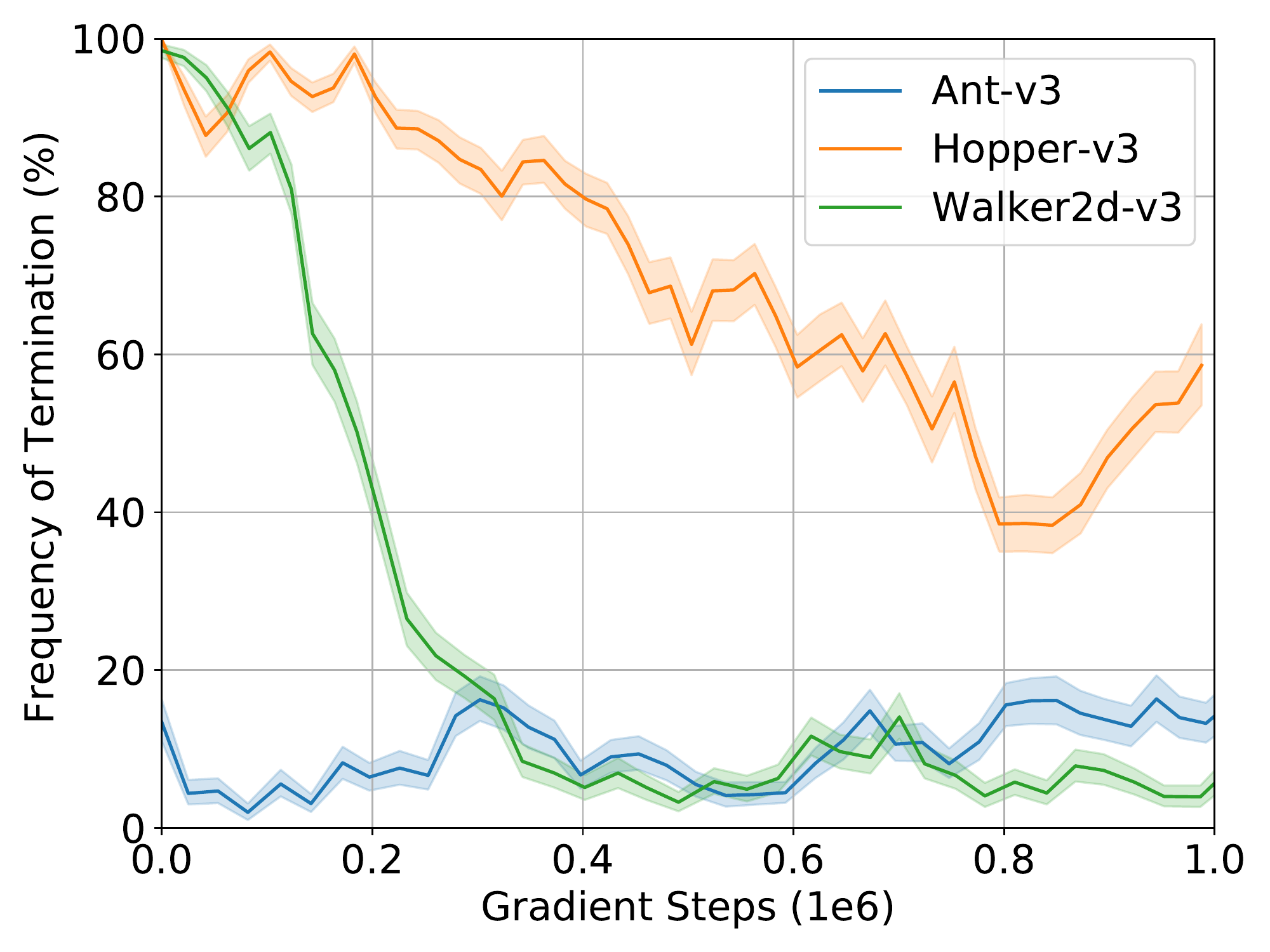}\label{fig:termination}}
    \subfigure[Performance without terminal]{\includegraphics[width=0.3\columnwidth]{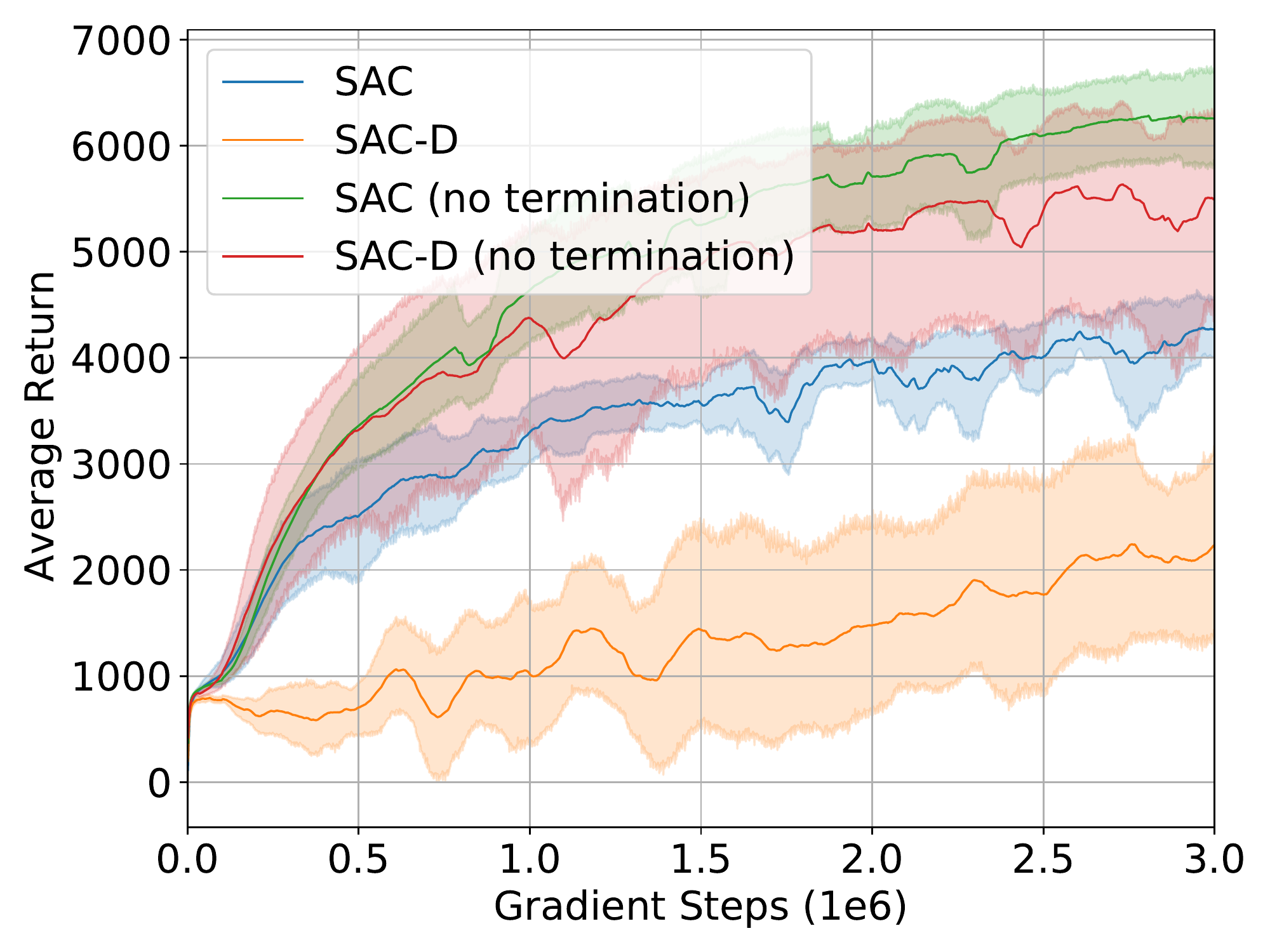}\label{fig:non_terminal}}
    \subfigure[Critic loss]{\includegraphics[width=0.3\columnwidth]{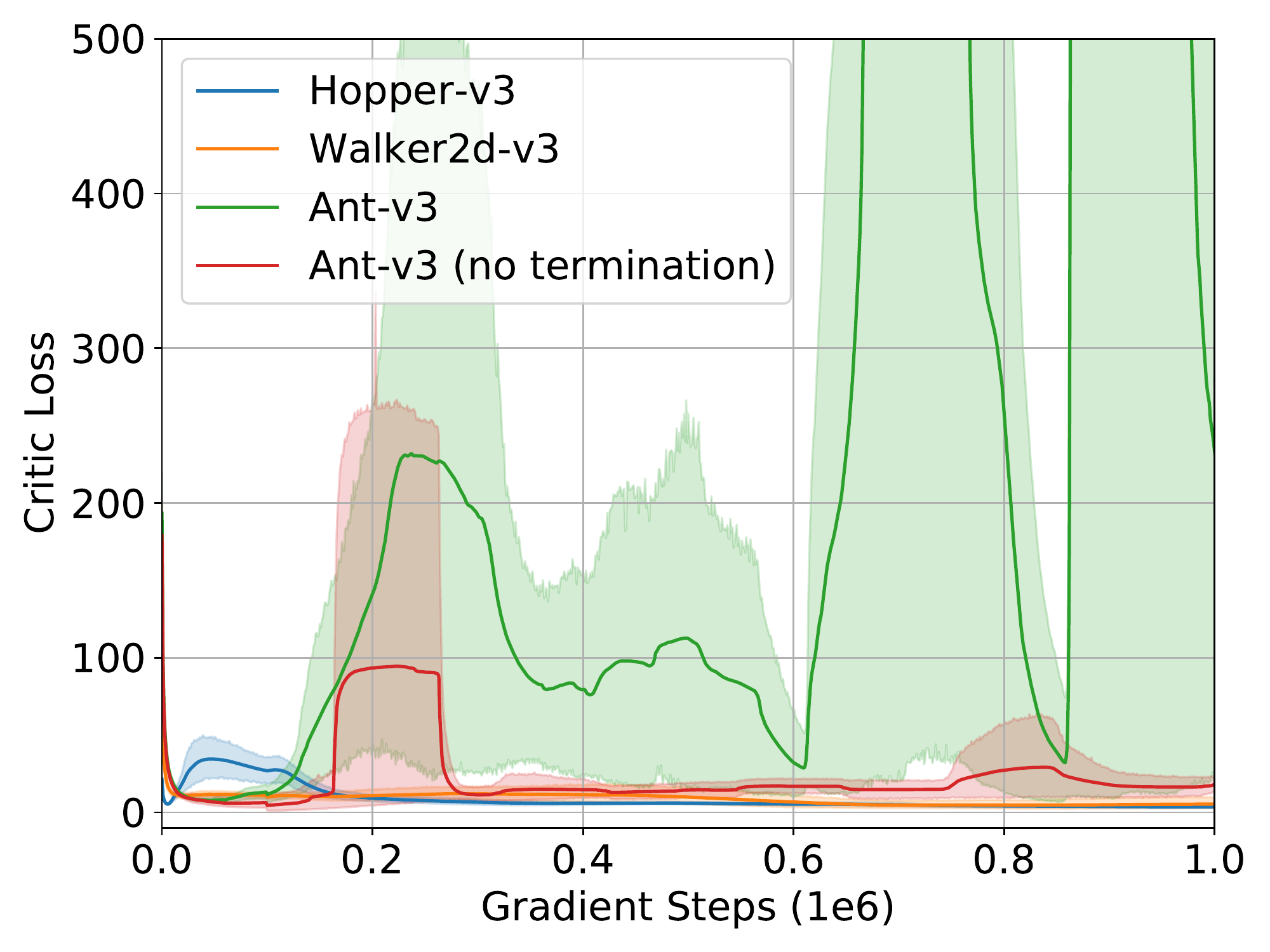}\label{fig:critic_loss}}
    \caption{\textbf{(a)} Frequency of environmental terminations during training. The metrics data is collected by training SAC-D for 1M steps and averaged over 10 trials. The shaded region represents the confidence interval. The lines are uniformly smoothed for visual clarity. There are significantly fewer terminated episodes in Ant-v3 \textbf{(b)} Training curves on Ant-v3 with no termination conditions. The performance gap between SAC-D and SAC is significantly smaller without environment termination. \textbf{(c)} Critic loss of SAC-D training. The critic loss in Ant-v3 is enormously large. By removing the terminations, the critic loss is substantially suppressed.}
\end{figure}
In our results, SAC-D performed similarly or only slightly worse than SAC on most environments, with the exception of Ant-v3 in which it performed significantly worse than SAC. While using CAGrad significantly improved results of SAC-D on Ant-v3, we wanted to further analyze why results were so different on this environment. 

From our investigation, we found that the abnormal behavior was due to poor component value predictions of rare environment terminations that led to large critic losses. Unlike other environments, in Ant, it is uncommon for even an untrained policy to terminate in failure. Consequently, terminating transitions were less well represented in the replay buffer for Ant than in other environments (\figref{fig:termination}.). Because terminating transitions constitute a large component value difference compared to non-terminating transitions (no more control cost, forward progress, nor health reward can be accrued), these rare instances in the replay buffer were associated with very large critic losses (and gradients) that destabilized learning. \figref{fig:critic_loss} shows that the average critic loss of SAC-D for Ant was much larger compared to other environments. It's possible SAC suffered less from this issue because the component rewards partially cancel out in the composite reward (e.g., negative control cost partially cancels positive forward progress when added together), resulting in a smaller loss.

To add further evidence that the difficulty was due to predicting the value of terminating transitions, we also trained SAC and SAC-D on a version on Ant that never terminates. In this case, we found that both SAC and SAC-D performed better than their respective terminating settings and that their performance with each other was much more comparable (\figref{fig:non_terminal}. We similarly found that the critic loss was much better behaved (\figref{fig:critic_loss}).

We believe the nature of CAGrad to re-weight component gradients to maximize the worst-component improvement stabilized results because if one component had an overly large error that conflicted with other components, CAGrad would decrease the weight of the gradient thereby keeping the entire process more stable. However, our results here suggest it might also be possible to use SAC-D without CAGrad when being careful about the definition of termination conditions, data representation, or other measures that make the value prediction task easier.

\section{Analysis of CAGrad behavior}
\label{appendix:cagrad_analysis}

\begin{figure}[ht]
    \centering
    \subfigure[Decomposed critic loss]{
        \subfigure{
            \includegraphics[width=0.45\columnwidth]{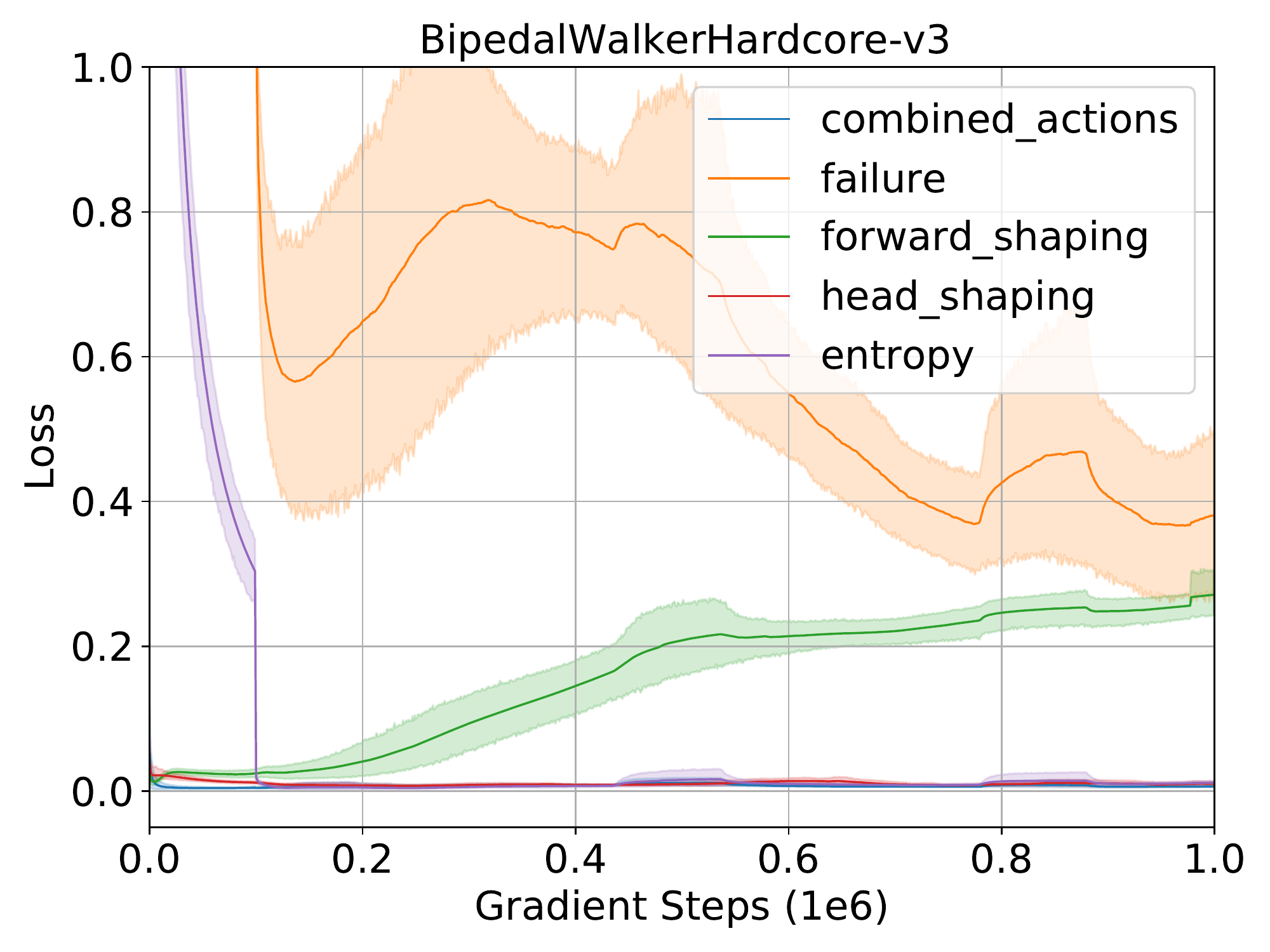}
            \addtocounter{subfigure}{-1}
        }
        \subfigure{
            \includegraphics[width=0.45\columnwidth]{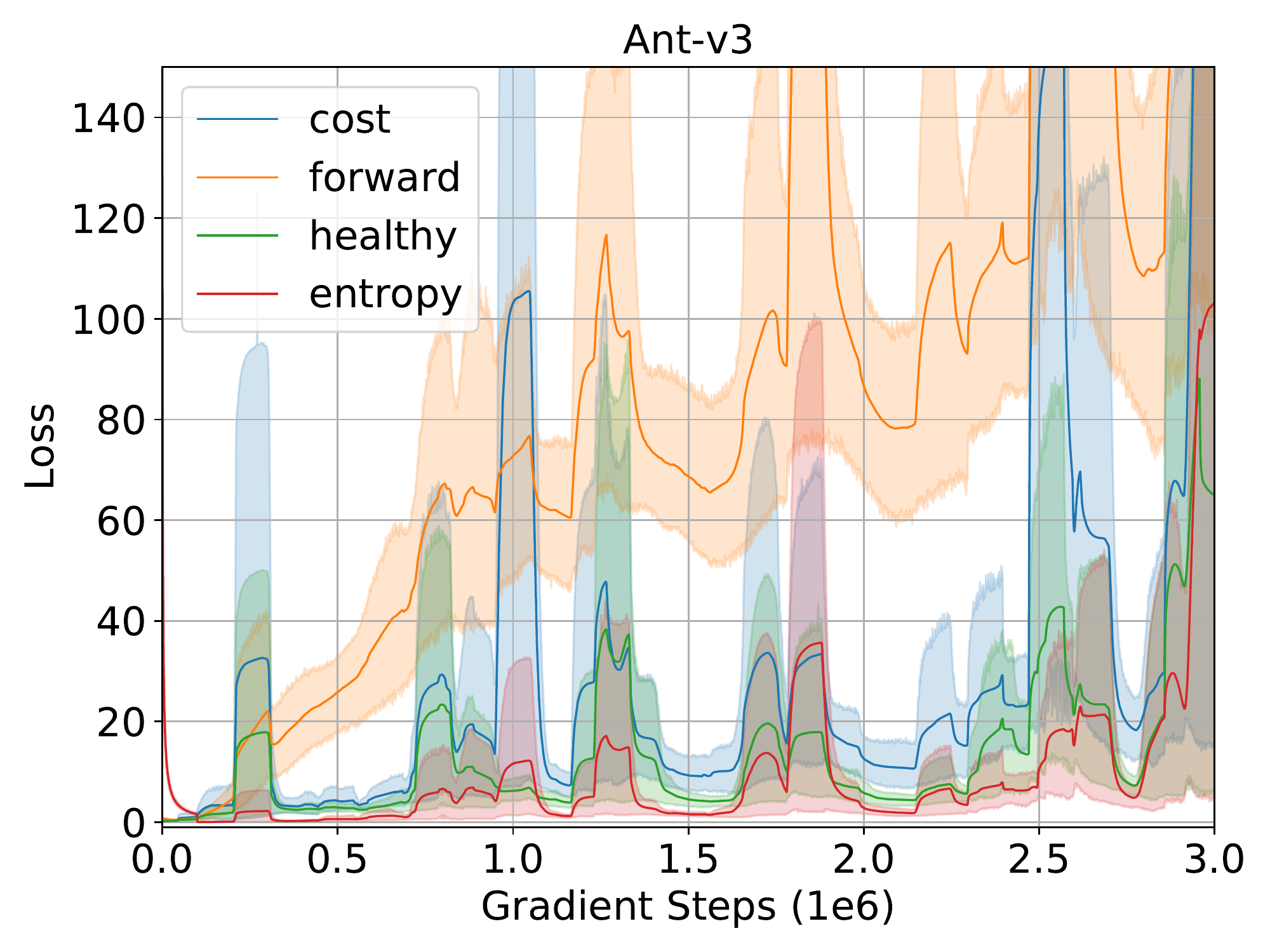}
            \addtocounter{subfigure}{-1}
        }
        \label{fig:factored_loss}
    }
    \subfigure[Optimized gradient weights]{
        \subfigure{
            \includegraphics[width=0.45\columnwidth]{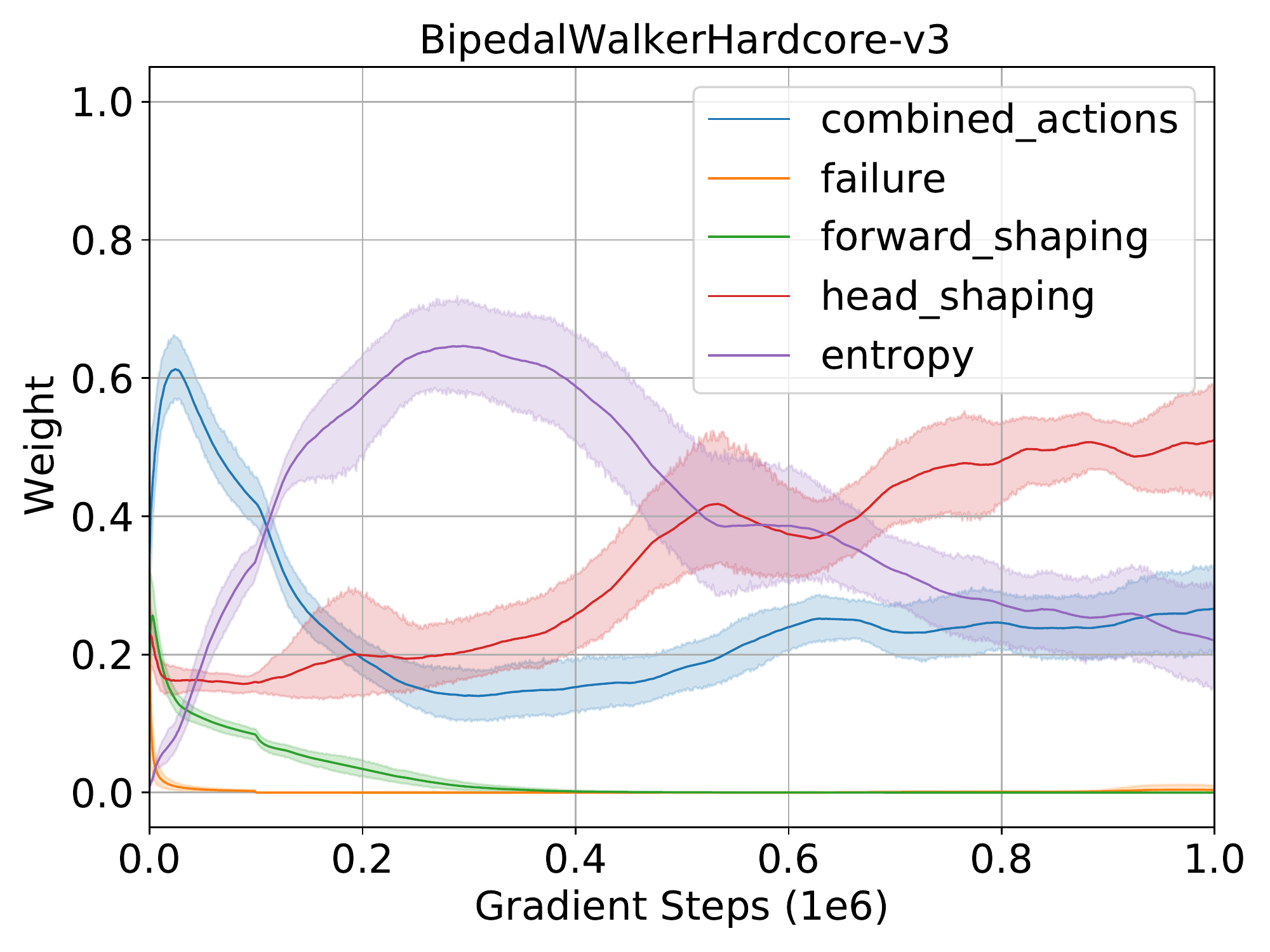}
            \addtocounter{subfigure}{-1}
        }
        \subfigure{
            \includegraphics[width=0.45\columnwidth]{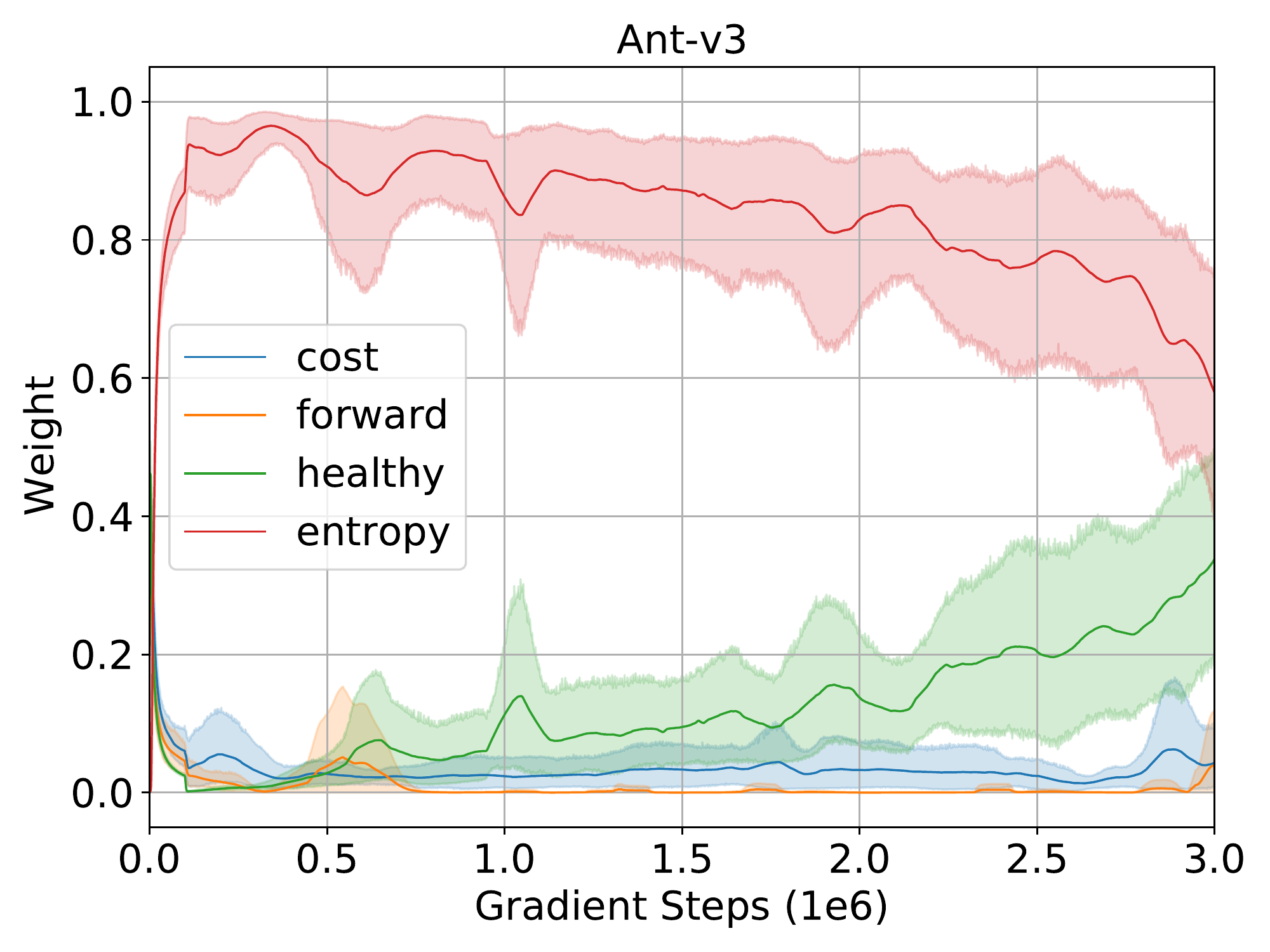}
            \addtocounter{subfigure}{-1}
        }
        \label{fig:cagrad_w}
    }
    \caption{\textbf{(a)} Decomposed critic loss during SAC-D-CAGrad training. The shaded regions represents the confidence interval. The left figure shows the decomposed critic loss of BipedalWalkerHardcore-v3, and the right figure shows the one of Ant-v3. \textbf{(b)} Weights of gradients optimized by CAGrad. Here we see that the component loss is inversely related to the weights CAGrad assigns to the component gradients.}
    \label{fig:cagrad_w_and_factored_loss}
\end{figure}

In the context of SAC-D, every gradient step, CAGrad optimizes the weighting vector of component Q-function gradients:
\begin{equation}
    \min_w F(w) := g_w ^\top g_0 + \sqrt{\phi} \|g_w\|, \text{where } g_w = \frac{1}{m+1} \sum^{m+1}_{i=1} w_i \nabla LQ_i(\theta)
    \label{eq:cagrad}
\end{equation}
where $\theta$ is a shared parameter, $g_0 = \frac{1}{m+1} \sum \nabla LQ_i(\theta)$, $\phi = c^2 \|g_0\|^2$, $c$ is a constant hyperparameter, and $\nabla LQ_i(\theta)$ is the gradient of the loss for the $i$th component Q-function with respect to Q-network parameters $\theta$. The solution to this optimization problem, $w^* = \argmin_w F(w)$, is used to derive the final gradient step direction: $\vec{\theta} \triangleq \sum_i^{m+1} w^*_i \nabla LQ_i(\theta)$.
Equation~\ref{eq:cagrad} can be interpreted as a function that balances the gradient magnitudes by assigning the small $w^*_i$ to the large gradients that conflict with other gradients.

The relationship of decomposed critic loss values and the result of the inner optimization is shown in \figref{fig:cagrad_w_and_factored_loss}.
In both BipedalWalkerHardcore-v3 and Ant-v3, the component Q-function losses are inversely related to the gradient weight CAGrad associates with the component gradients: the higher the critic loss, the lower the weight CAGrad tended to assign the component.
This result shows CAGrad prevents particularly large component errors from dominating the entire process, and helps explain why CAGrad effectively combated the loss of performance in Ant that was due to particularly large errors for rare terminating transitions. We expect CAGrad will be particularly useful in domains where different reward components induce returns with very different magnitudes.

\section{Influence}
\label{appendix:influence}

Here we provide fractional influence plots for the various environments studied. For each plot the factors are sorted by the fractional influence at the very last gradient step with larger influence towards the bottom.

\begin{figure}[ht]
    \centering
    \subfigure[SAC-D]{\includegraphics[width=0.49\columnwidth]{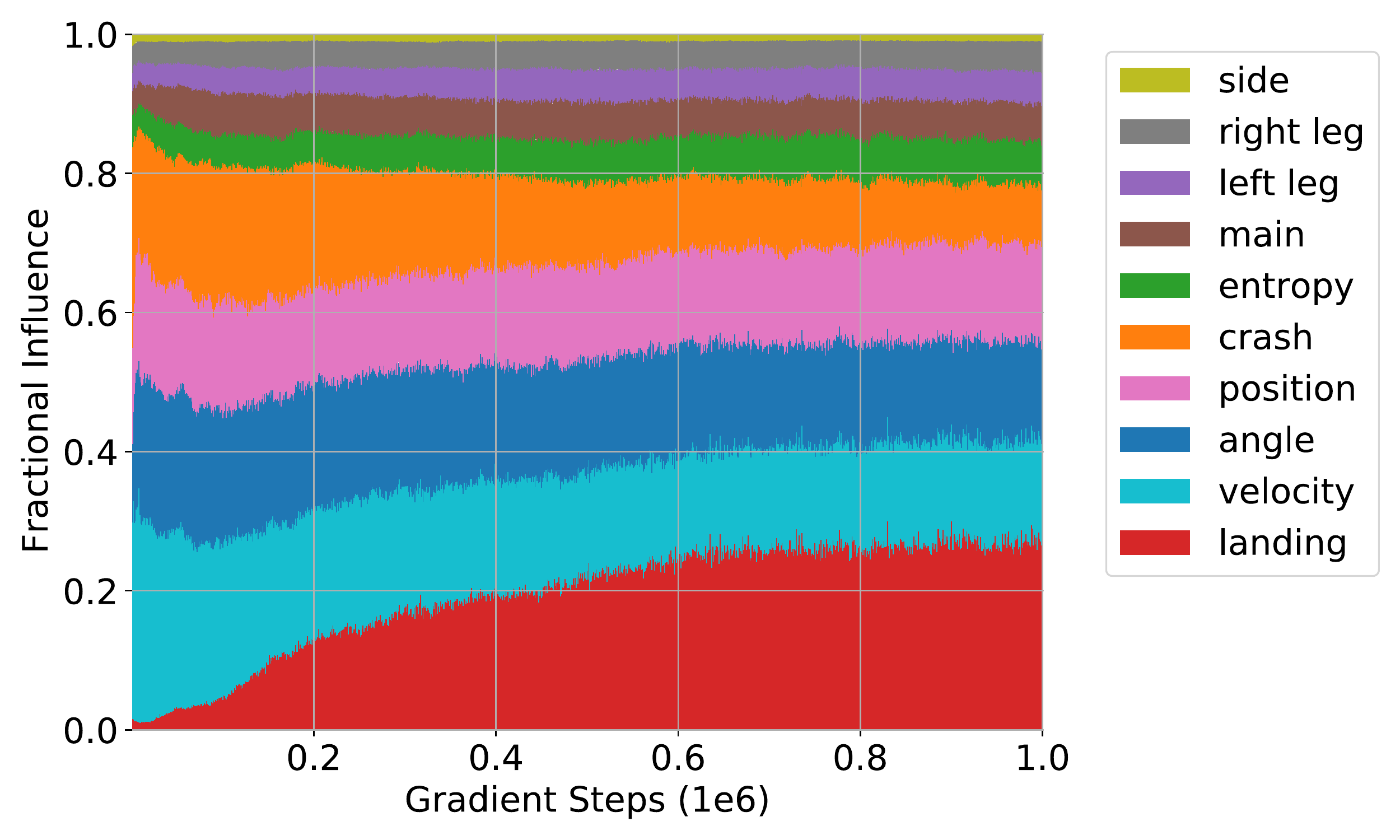}}
    \subfigure[SAC-D-CAGrad]{\includegraphics[width=0.49\columnwidth]{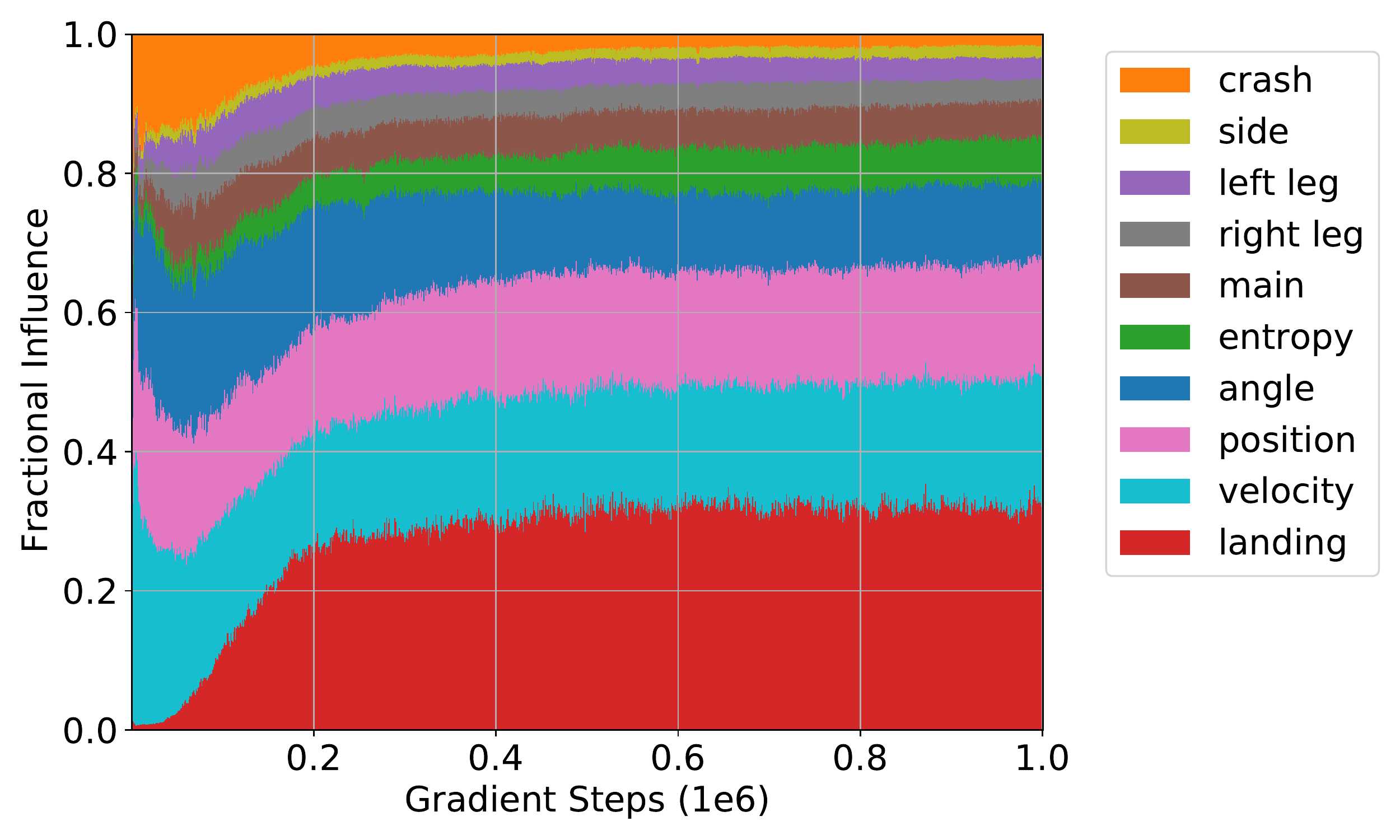}}
    \caption{Lunar Lander}
\end{figure}

\begin{figure}[ht]
    \centering
    \subfigure[SAC-D]{\includegraphics[width=0.49\columnwidth]{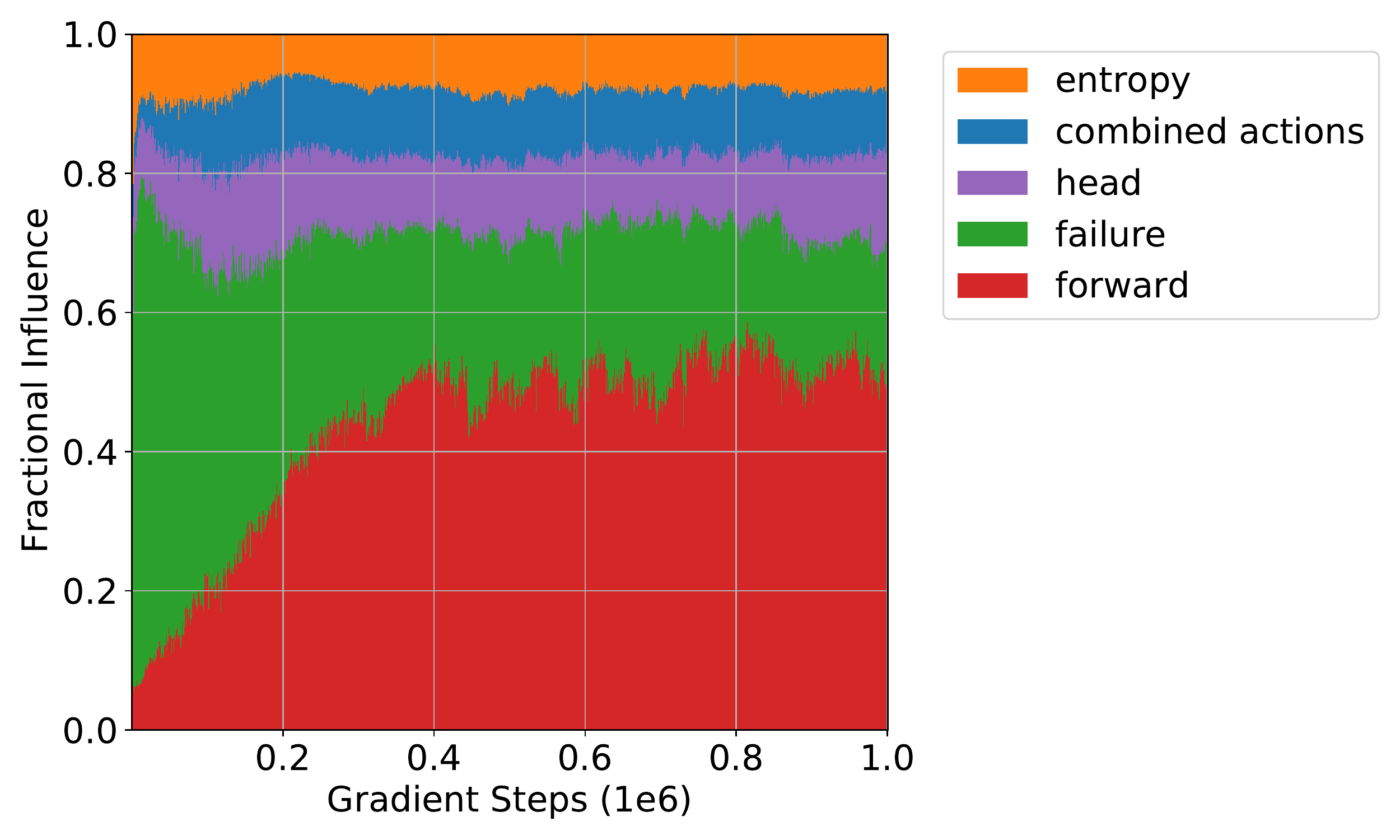}}
    \subfigure[SAC-D-CAGrad]{\includegraphics[width=0.49\columnwidth]{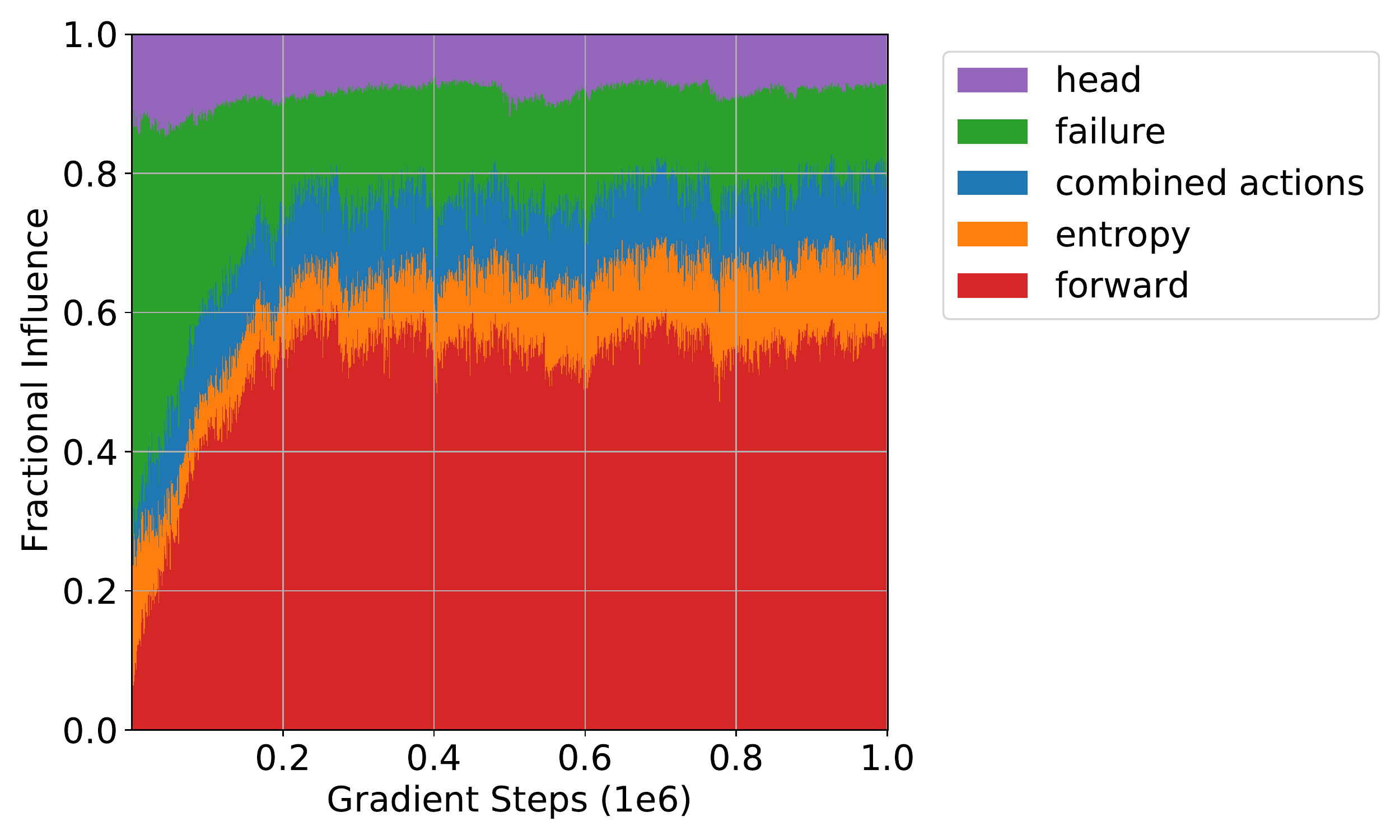}}
    \caption{Bipedal Walker}
\end{figure}

\begin{figure}[ht]
    \centering
    \subfigure[SAC-D]{\includegraphics[width=0.49\columnwidth]{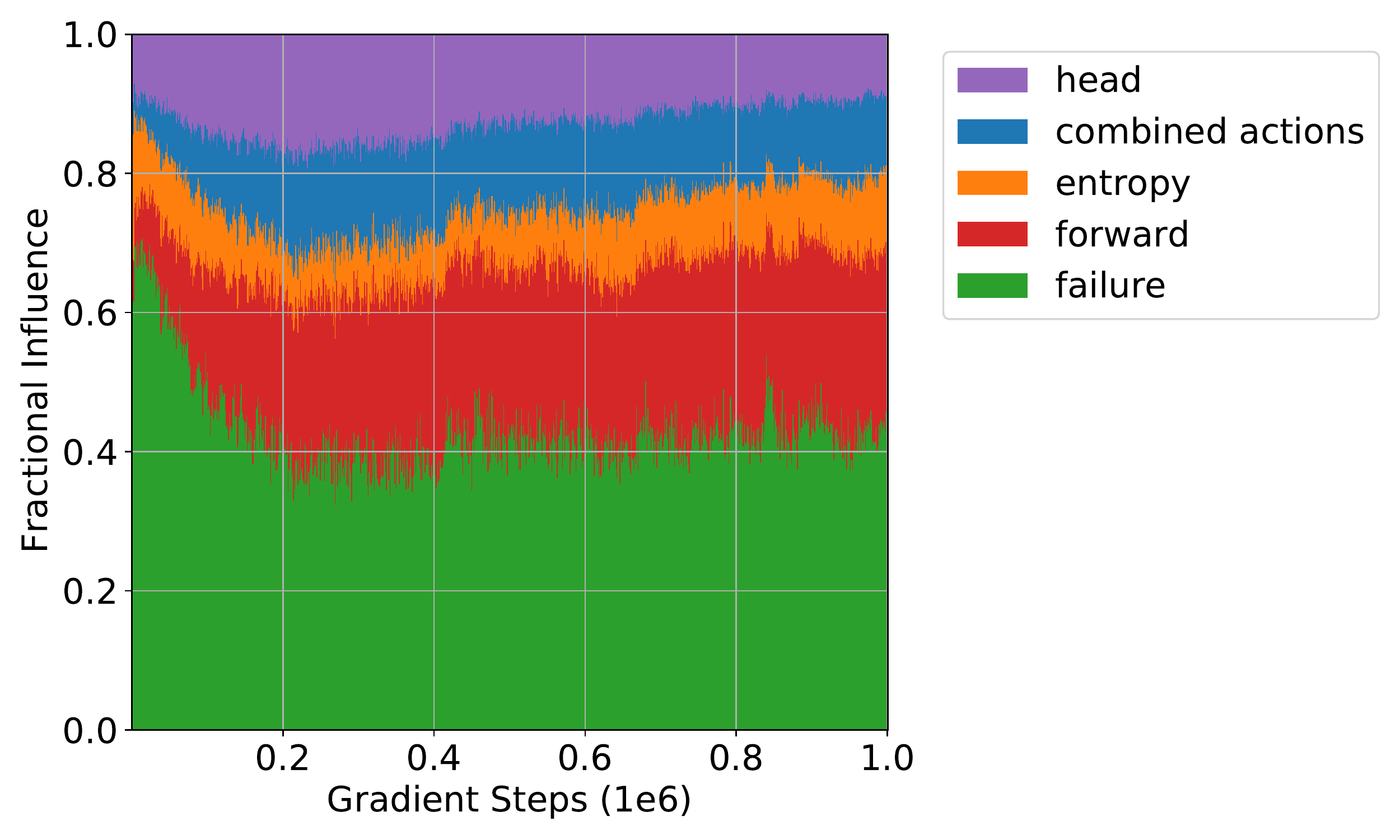}}
    \subfigure[SAC-D-CAGrad]{\includegraphics[width=0.49\columnwidth]{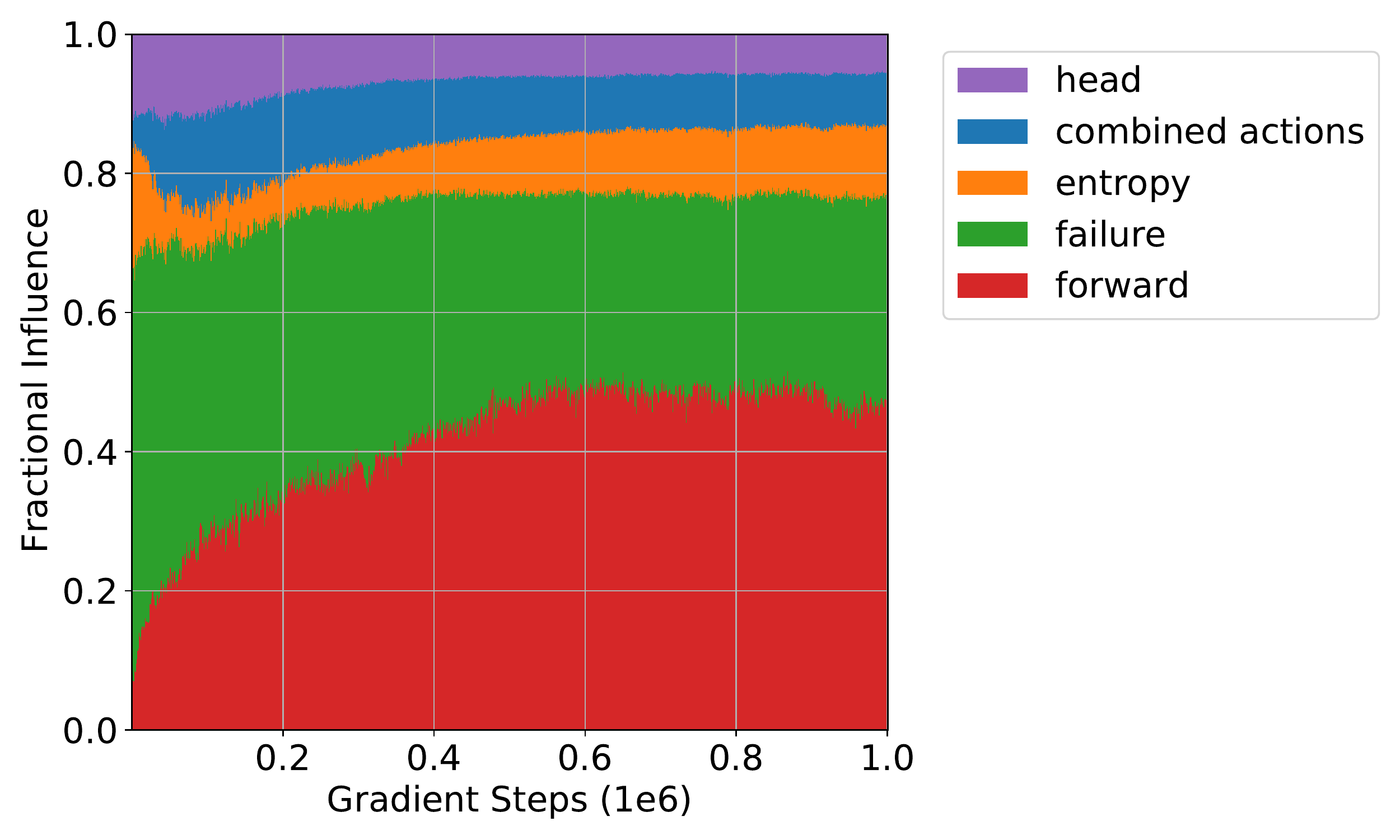}}
    \caption{Bipedal Walker Hardcore}
\end{figure}

\begin{figure}[ht]
    \centering
    \subfigure[SAC-D]{\includegraphics[width=0.49\columnwidth]{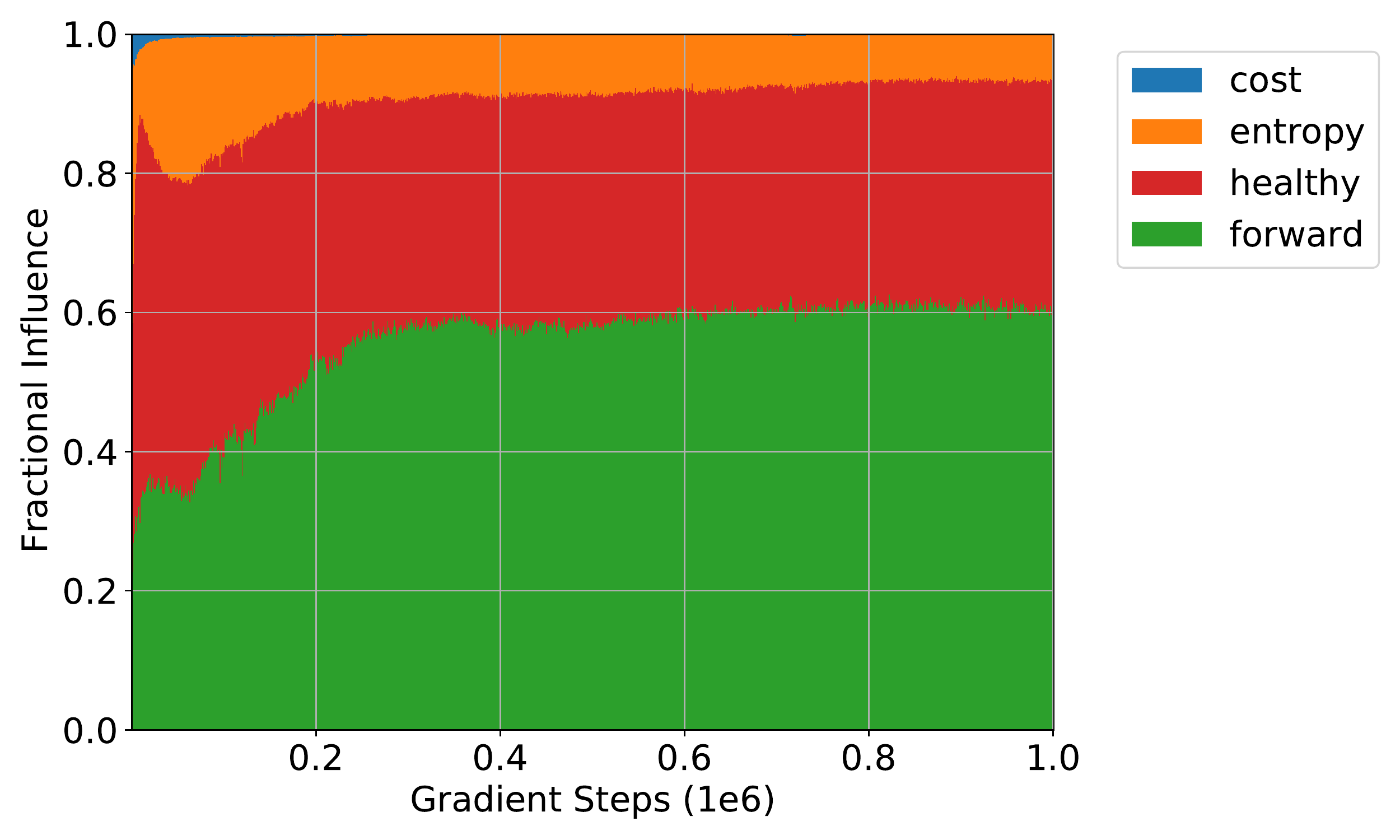}}
    \subfigure[SAC-D-CAGrad]{\includegraphics[width=0.49\columnwidth]{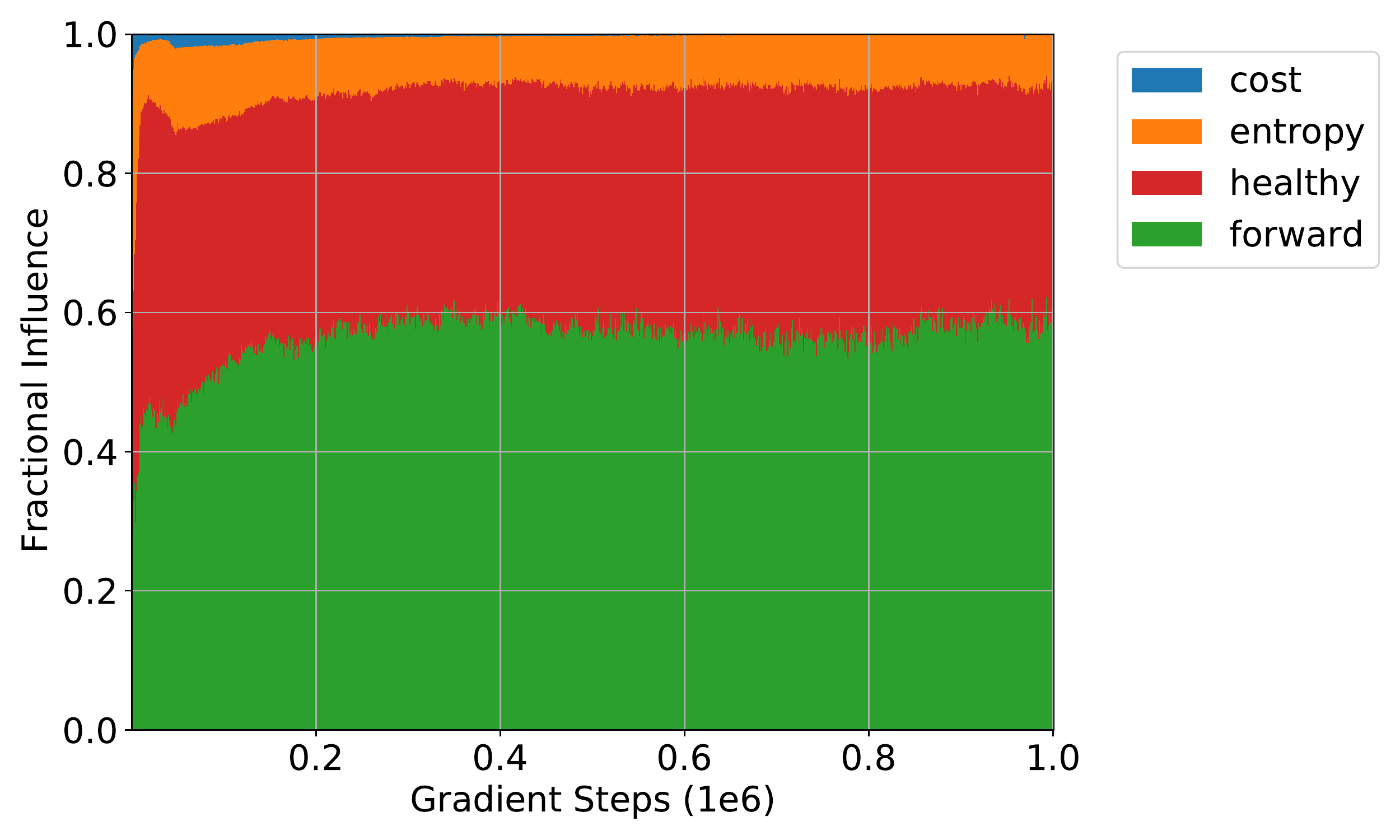}}
    \caption{Hopper}
\end{figure}

\begin{figure}[ht]
    \centering
    \subfigure[SAC-D]{\includegraphics[width=0.49\columnwidth]{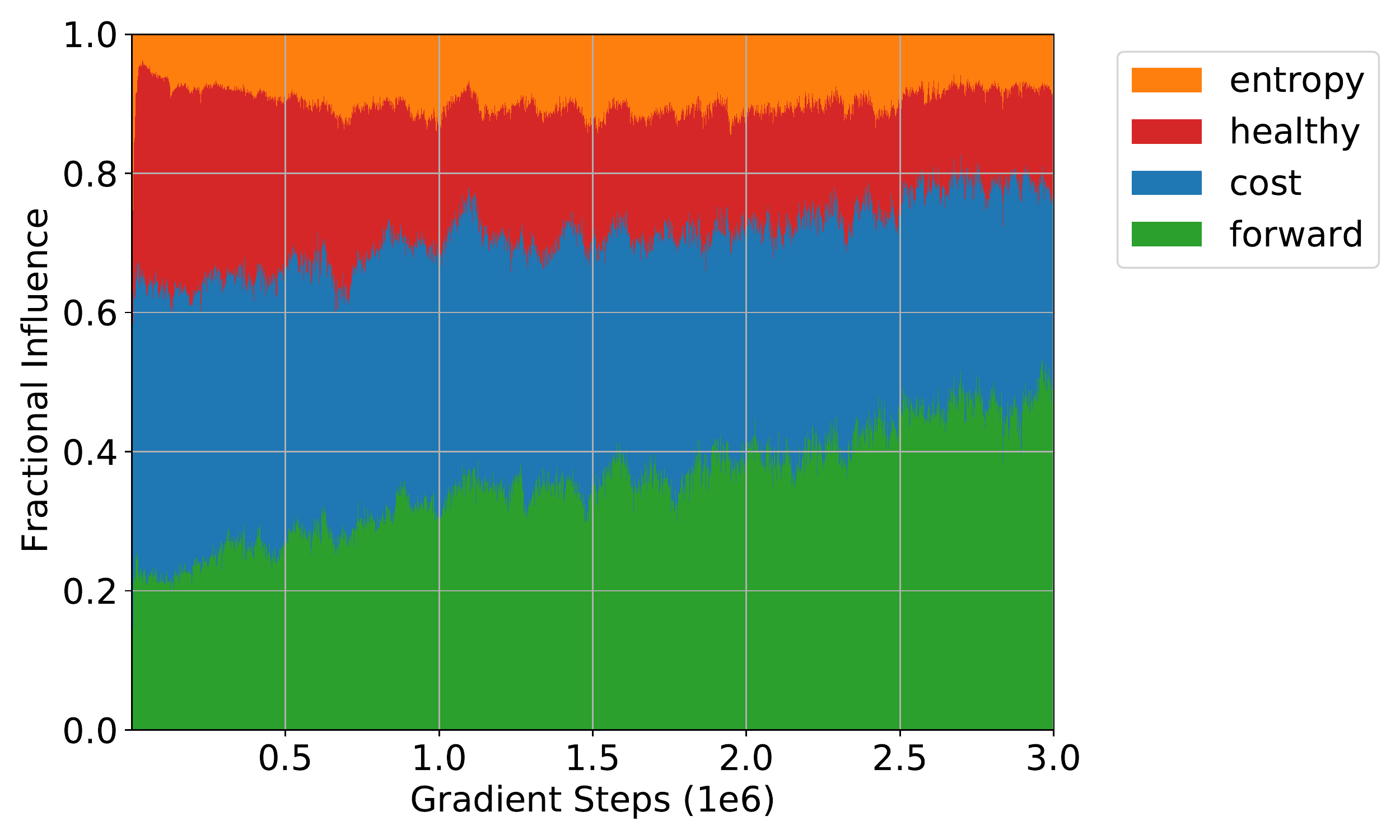}}
    \subfigure[SAC-D-CAGrad]{\includegraphics[width=0.49\columnwidth]{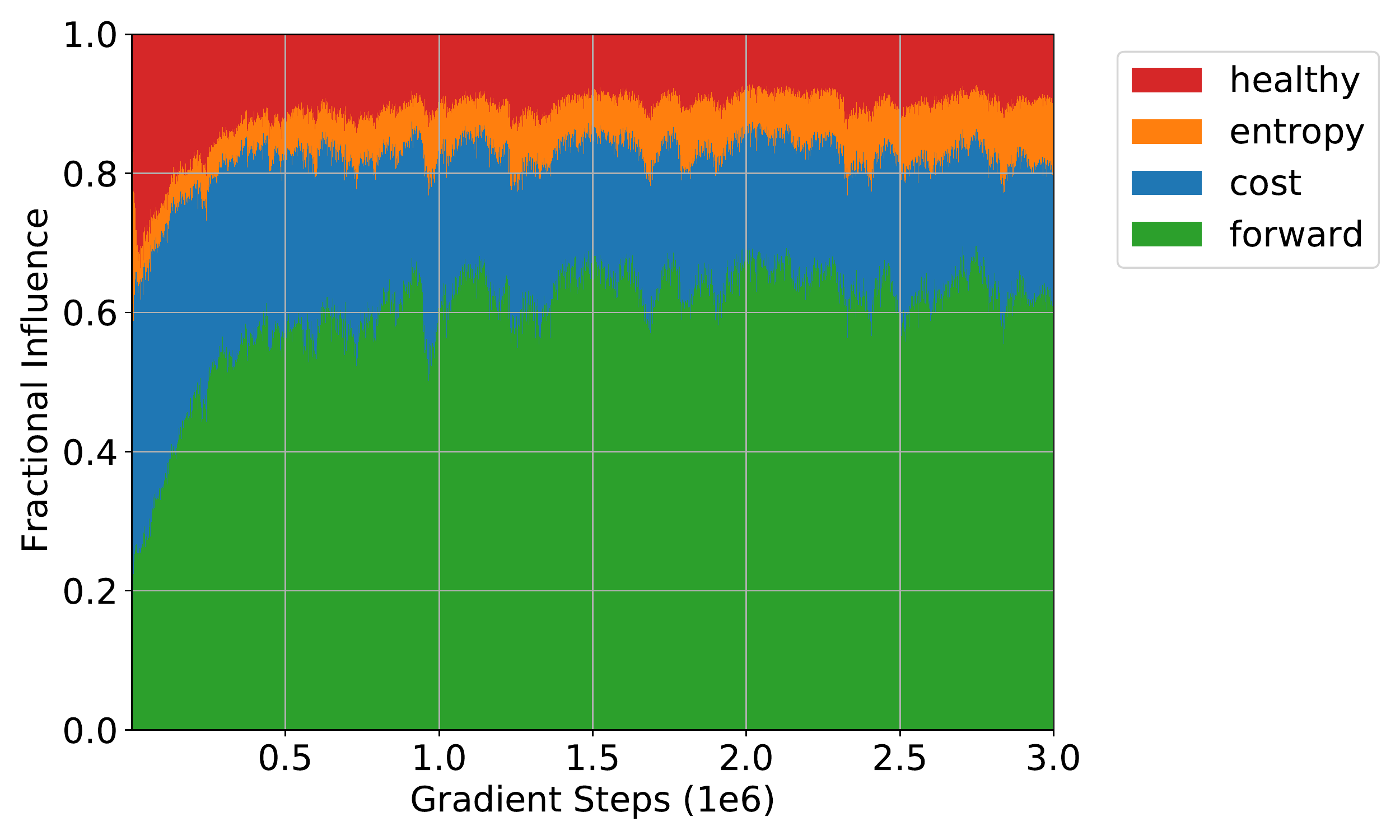}}
    \caption{Ant}
\end{figure}

\begin{figure}[ht]
    \centering
    \subfigure[SAC-D]{\includegraphics[width=0.49\columnwidth]{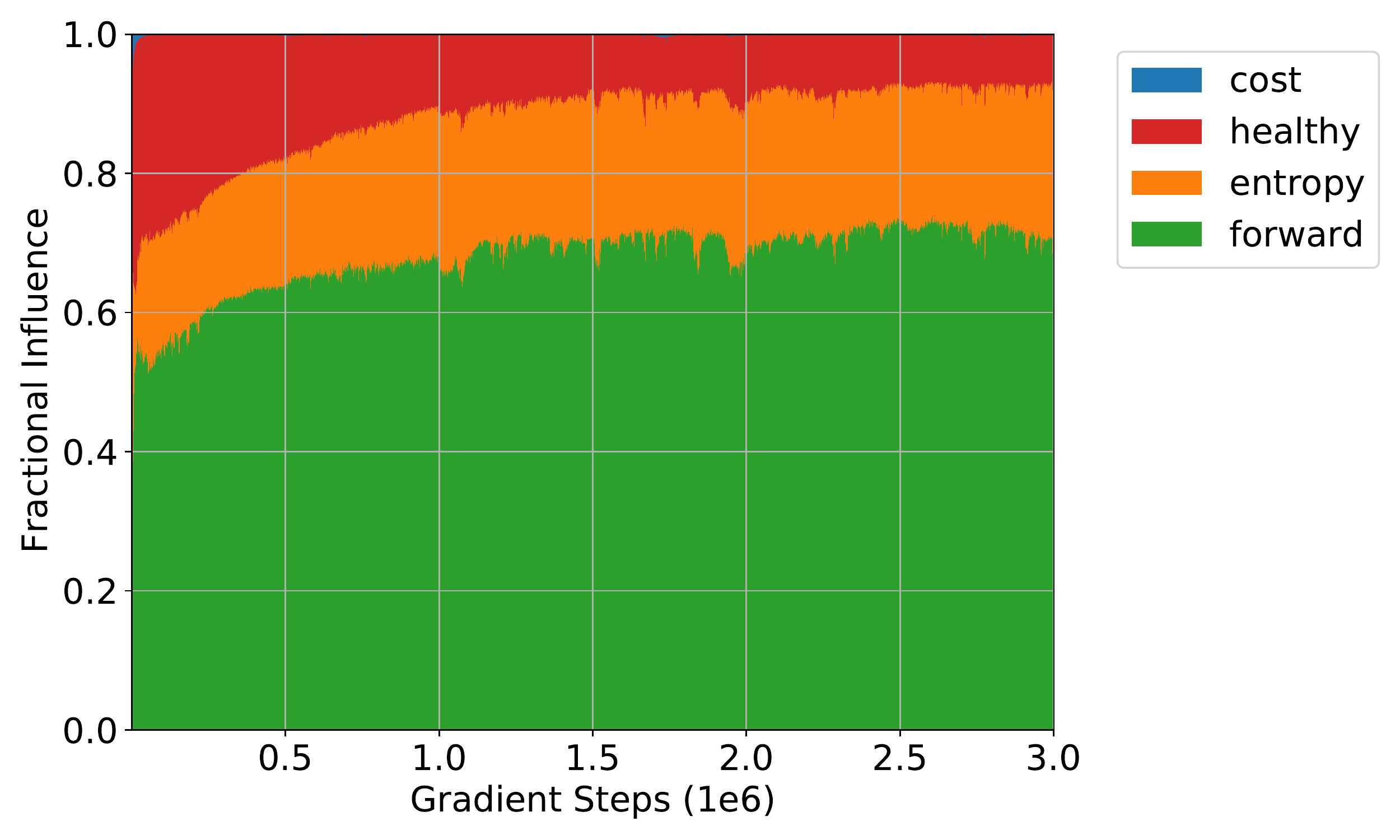}}
    \subfigure[SAC-D-CAGrad]{\includegraphics[width=0.49\columnwidth]{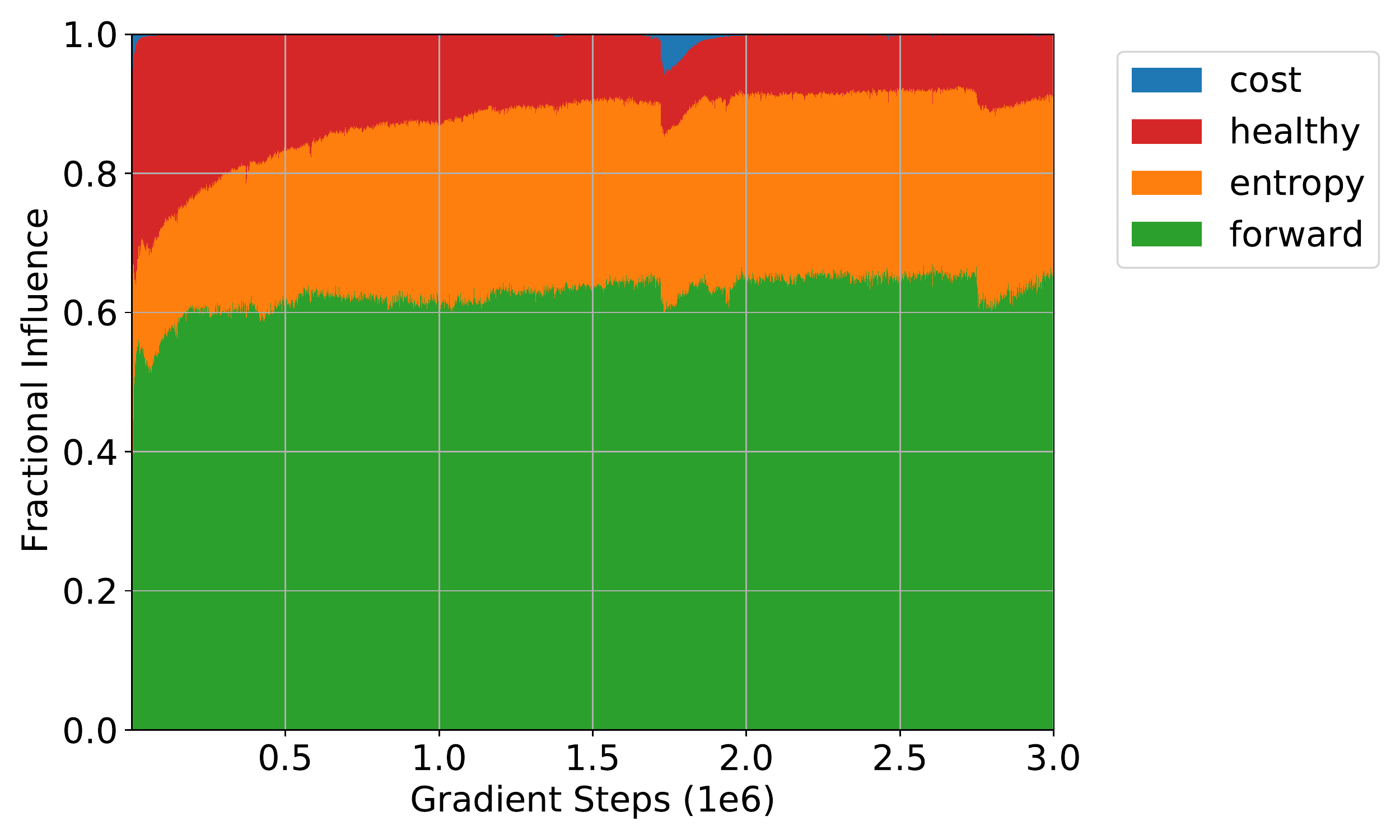}}
    \caption{Walker 2D}
\end{figure}

\begin{figure}[ht]
    \centering
    \subfigure[SAC-D]{\includegraphics[width=0.49\columnwidth]{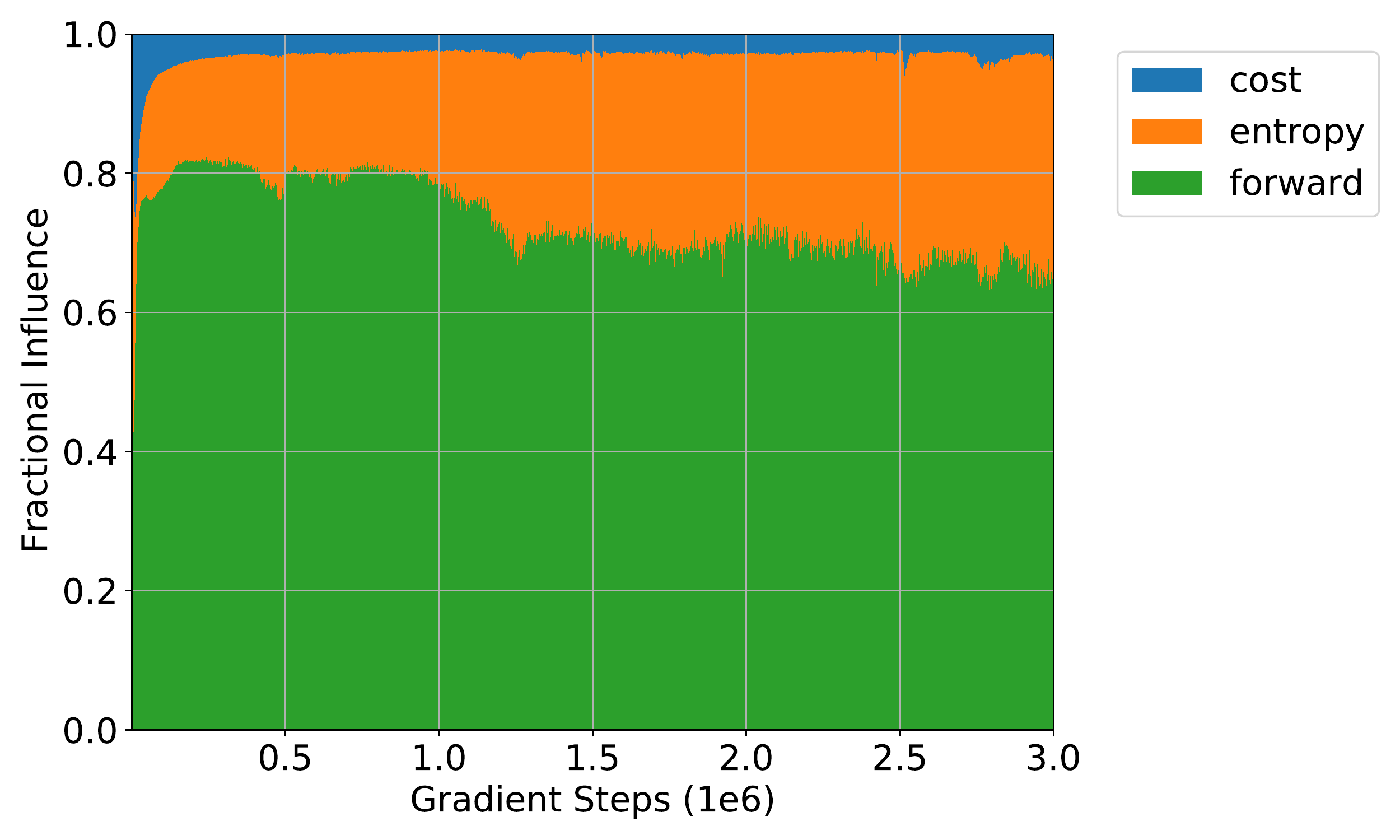}}
    \subfigure[SAC-D-CAGrad]{\includegraphics[width=0.49\columnwidth]{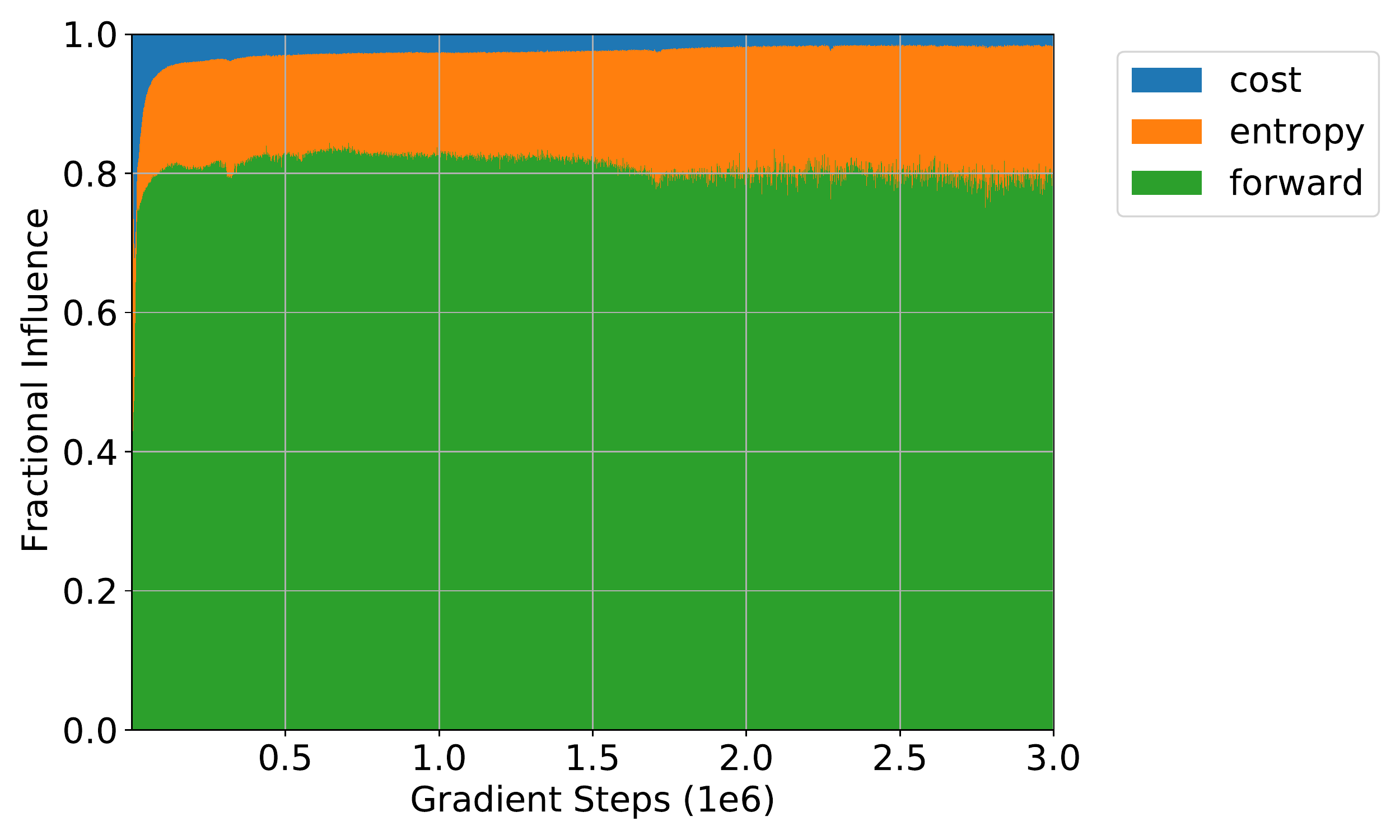}}
    \caption{Half Cheetah}
\end{figure}

\begin{figure}[ht]
    \centering
    \subfigure[SAC-D]{\includegraphics[width=0.49\columnwidth]{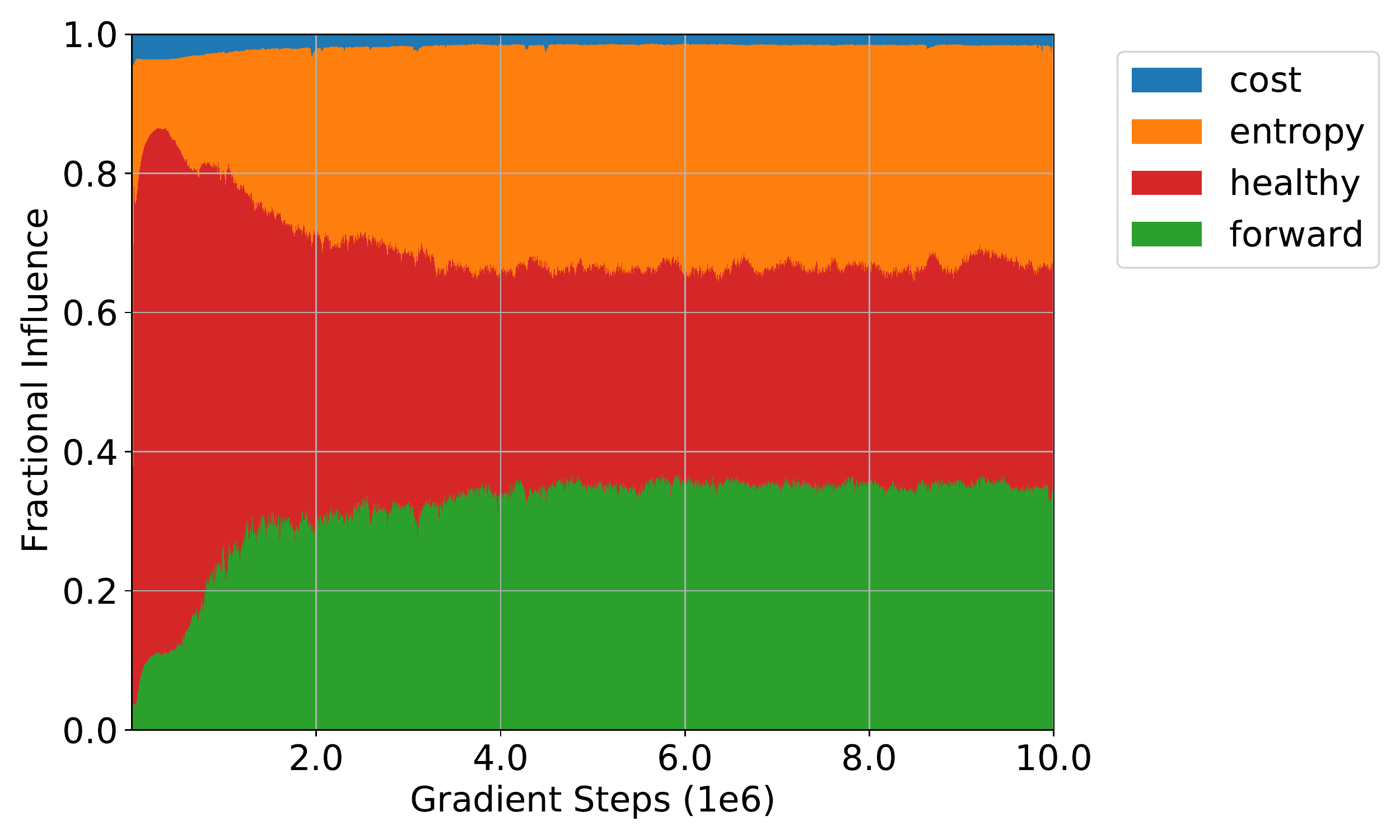}}
    \subfigure[SAC-D-CAGrad]{\includegraphics[width=0.49\columnwidth]{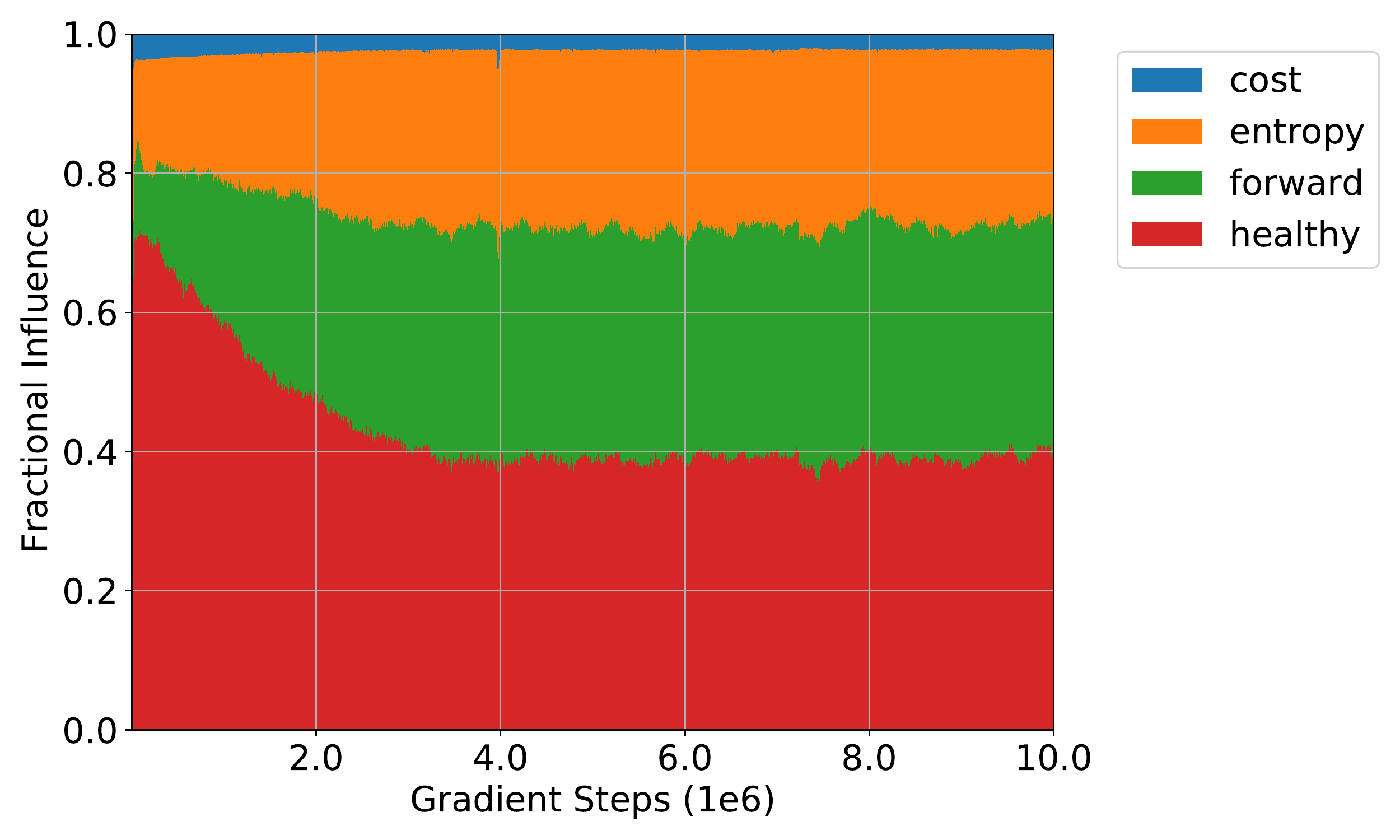}}
    \caption{Humanoid}
\end{figure}

\clearpage

\section{Markov features}
\label{appendix:markov}
In Sec.~\ref{sec:analysis_examples}, we showed that the Lunar Lander observation features were insufficient to properly predict when landing would occur, because the simulator only signaled successful landing after the agent had been stable on the ground for a certain number of time steps. Here we provide the description of the additional \textit{zero velocity trace feature} we added to the environment to make it Markov.
We define the zero velocity trace feature as: $V_{0}^\text{trace}(t) = V_{0}^{\text{steps}}(t)/c$, where $V_{0}^{\text{steps}}(t)$ is the number of time steps since all the velocities dropped below a threshold and $c$ is a fixed normalizing constant. In \figref{fig:ll_predictions} we use a value $c=40$. Velocity was considered to be zero when all velocity fields (horizontal, vertical and rotational) fell below 1e-3.

We compared the predictive accuracy with and without the $V_0^{\text{trace}}(t)$ feature. Using a policy trained with the baseline features we collected a dataset of 2000 trajectories (using the exploration policy). Next we trained two new value function models to predict the landing return (independent of all other factors), one for the baseline features and the other including the $V_0^{\text{trace}}(t)$ feature. For each time step, we computed the discounted Monte Carlo return ($\gamma=0.99$). Note that we did not use bootstrapping in the return calculations. This could have caused inaccuracies in the targets as some trajectories timed out rather than terminating in either a crash or a landing. Each model was then trained for 100 epochs, with each epoch consisting of randomly subdividing time steps from the 2000 trajectories into batches of 256 time steps each and training on each subdivision. We used the mean squared loss and computed updates using the ADAM optimizer with a learning rate of $0.001$.
We then collected an additional 200 trajectories in which the agent successfully landed (again, using the exploration policy). For these trajectories we computed the value estimates and returns. We then computed the interquartile mean (IQM) values of correlation and RMSE for the last 25 time steps of each trajectory (indicated by the dashed pink line in \figref{fig:ll_predictions}). From results in Table~\ref{tab:ll_markov}) we see a noticeable improvement in prediction accuracy. In particular, we see that the correlation is markedly improved, confirming that our predictions match the shape of the returns better in this time span.

\begin{table}[ht]
    \centering
    \caption{Landing Prediction Accuracy}    
    \begin{tabular}{@{}lrr@{}} \toprule
        Features & RMSE (std) & Correlation (std) \\
        \midrule
        Baseline & 5.61 (0.12) & 0.521 (0.030) \\
        Markov & \textbf{1.05 (0.28)} & \textbf{0.996 (0.001)}\\
        \bottomrule
    \end{tabular}
    \label{tab:ll_markov}
\end{table}

\section{Constraint experiment details}
\label{app:constraint}

The value function constraint experiment described in Sec.~\ref{sec:analysis_examples} and used to produce \figref{fig:ll_constraints} followed the same training procedures as the robustness experiments described in App.~\ref{appendix:detail}. However, the poor behavior, in which the unconstrained value estimates crossed zero, was most pronounced when we lowered the replay buffer size to 20k. The results shown \figref{fig:ll_constraints} were trained under this condition. 

Constraints can be applied in at least two ways. One is to constrain the component Q-value target by clipping. This prevents the update from pulling the predictions towards values that are invalid. Such a situation might result due to invalid bootstrapping estimates of the network. For example, the crash component in lunar lander can only be non-positive. Consequently, we constrain the target values as follows:
\begin{align*}
    y_{crash} = \min{\bigg(0, r_{crash} + \gamma Q_{crash}(s', a')\bigg)}.
\end{align*}

In addition to constraining the bootstrap target, we can also penalized the component value function estimate for being positive with the following additional loss term:
\begin{align*}
    L_{+;crash}=0.5\cdot \mathbbm{1}_{\mathbb{R}^+}(Q_{crash}(s, a)) \cdot Q_{crash}(s, a).
\end{align*}

Both these constraints have analogs for non-negative components. In the results shown in \figref{fig:ll_constraints}, we show only the effect on the \textit{crash} component (which benefited the most), but we also applied constraints to the \textit{side}, \textit{main}, and \textit{landing} components.

\section{Weight scheduling details}
\label{appendix:annealing}
In Sec.\ref{sec:analysis_examples}, we schedule the component weight of the Bipedal Walker Hardcore \textit{failure} reward from zero to one to remedy its exploration inhibiting properties. Specifically, we used the schedule function:
\begin{equation}
w_{\text{failure}}(t) = \tanh\left(\beta\left|t-\tau^\text{warmup}\right|_+\right),
\end{equation} 
where $t$ is the number of gradient steps taken; $\medmath{\left|x\right|_+ = \left\{\begin{array}{l} 
0.01, \,x < 0\\ 
x, \,x \geq 0
\end{array}\right.}$; $\tau^\text{warmup}$ is the number of gradient steps required before the value begins increasing; and $\beta$ controls the speed at which the weight returns to 1 after the warm-up period. We used $\tau^\text{warmup}=100$ and $\beta=0.0004$. With these values, the \textit{failure} component weight came close to its original value of 1 after one million gradient steps, as shown in \figref{fig:scheduling_rate}. 

\begin{figure}[ht]
\centering
\includegraphics[width=0.5\columnwidth]{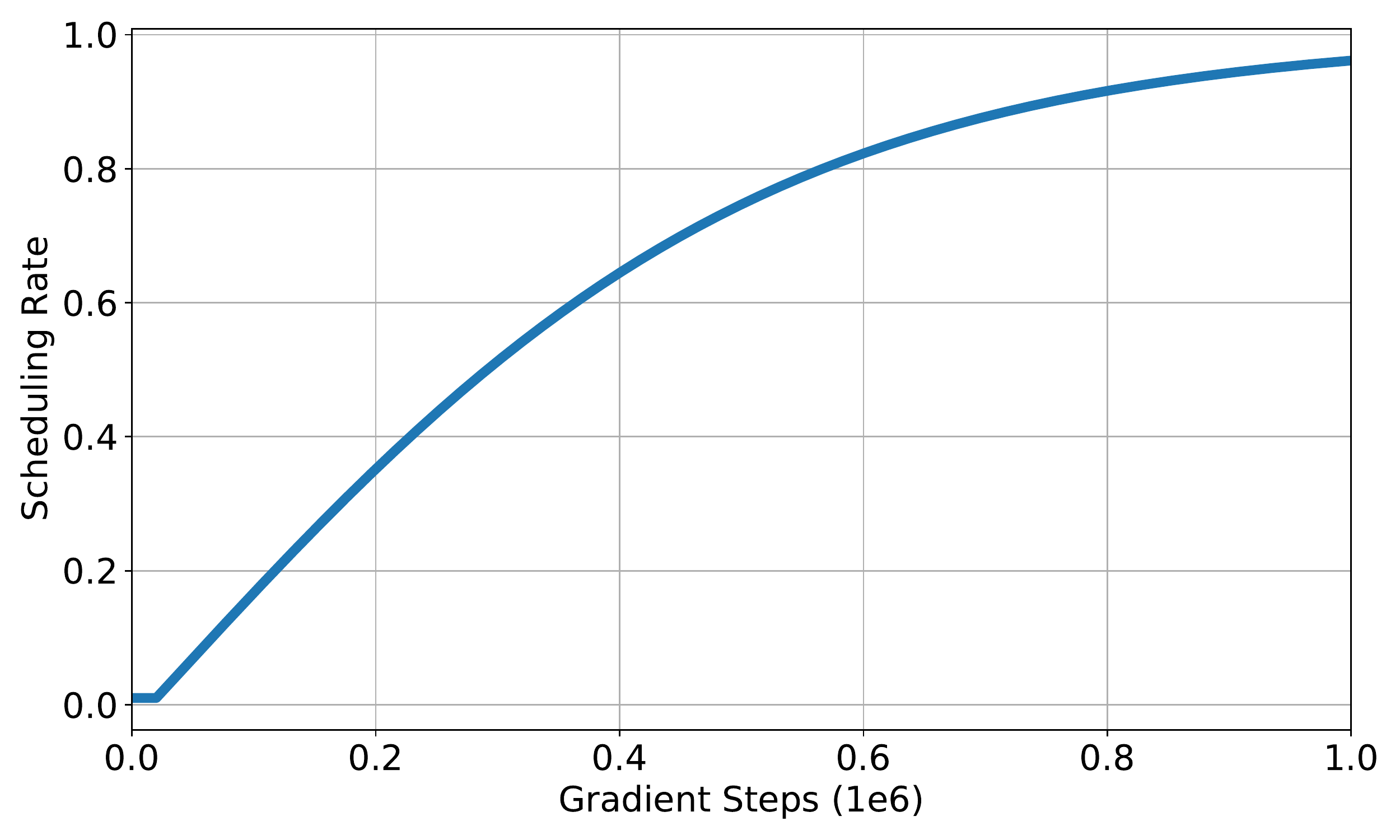}
\caption{Scheduling rate function for reverse annealing}
\label{fig:scheduling_rate}
\end{figure}

The results of these \textit{failure} component weight scheduling experiments are shown in the score distribution plot of \textit{forward} progress per episode score in \figref{fig:bwh_score_dist} and compared to the baseline (constant weight scenario). Additionally, in \figref{fig:bwh_annealing_influence} we see the effect that weight scheduling has on the influence of the \textit{failure} and \textit{forward} reward components. With \textit{failure}'s initial influence reduced, \textit{forward} is able to dominate the policy behavior from early on.

\begin{figure}[ht]
    \centering
    \subfigure[Constant weights]{\includegraphics[width=0.49\columnwidth]{figures/influence/bwh_cagrad_influence.pdf}}
    \subfigure[Reverse annealing]{\includegraphics[width=0.49\columnwidth]{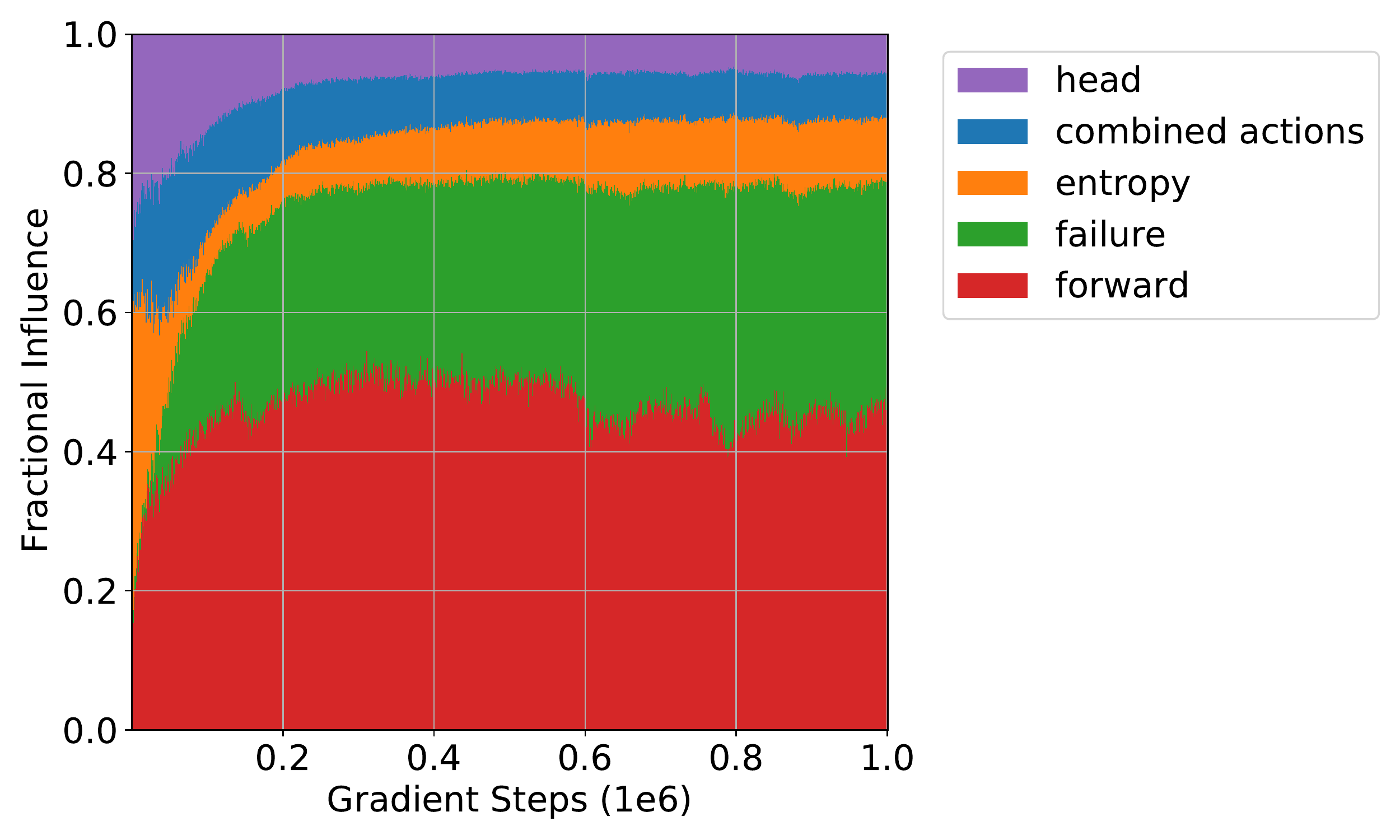}}  
    \caption{Figures show the impact of reverse annealing on learning with SAC-D-CAGrad in the Bipedal Walker Hardcore environment. (\textbf{a}) When a constant weighting of 1 is applied to all components \textit{failure} dominates early learning. (\textbf{b}) Reverse annealing the weight on \textit{failure} allows \textit{forward} to drive policy behavior from early on.}
    \label{fig:bwh_annealing_influence}
\end{figure}

\end{document}